\documentclass{article}

\usepackage[utf8]{inputenc} \usepackage[T1]{fontenc}    \usepackage{hyperref}       \usepackage{url}            \usepackage{booktabs}       \usepackage{amsfonts}       \usepackage{nicefrac}       \usepackage{microtype}      
\usepackage[margin=1.125in]{geometry}
\usepackage[round]{natbib}
\usepackage{wrapfig}
\usepackage[normalem]{ulem} \usepackage[makeroom]{cancel} \usepackage[position=b]{subcaption}
\usepackage{graphicx}
\usepackage{amsmath,amsfonts,amssymb,amsthm,commath,dsfont}
\usepackage{enumitem}
\usepackage{xspace}
\usepackage{cleveref}
\usepackage[dvipsnames]{xcolor}
\usepackage{thmtools}
\makeatletter
\newcommand{\opnorm}{\@ifstar\@opnorms\@opnorm}
\newcommand{\@opnorms}[1]{  \left|\mkern-1.5mu\left|\mkern-1.5mu\left|
   #1
  \right|\mkern-1.5mu\right|\mkern-1.5mu\right|
}
\newcommand{\@opnorm}[2][]{  \mathopen{#1|\mkern-1.5mu#1|\mkern-1.5mu#1|}
  #2
  \mathclose{#1|\mkern-1.5mu#1|\mkern-1.5mu#1|}
}
\makeatother

\def\ddefloop#1{\ifx\ddefloop#1\else\ddef{#1}\expandafter\ddefloop\fi}
\def\ddef#1{\expandafter\def\csname bb#1\endcsname{\ensuremath{\mathbb{#1}}}}
\ddefloop ABCDEFGHIJKLMNOPQRSTUVWXYZ\ddefloop
\def\ddef#1{\expandafter\def\csname c#1\endcsname{\ensuremath{\mathcal{#1}}}}
\ddefloop ABCDEFGHIJKLMNOPQRSTUVWXYZ\ddefloop
\def\ddef#1{\expandafter\def\csname h#1\endcsname{\ensuremath{\widehat{#1}}}}
\ddefloop ABCDEFGHIJKLMNOPQRSTUVWXYZ\ddefloop

\DeclareMathOperator*{\argmax}{arg\,max}

\newcommand{\eps}{\ensuremath{\epsilon}}
\newcommand{\veps}{\ensuremath{\varepsilon}}

\def\bfe{\mathbf{e}}
\def\Pr{\textup{Pr}}

\def\R{\mathbb{R}}
\def\1{\mathds{1}}

\newcommand{\ip}[2]{\left\langle #1, #2 \right \rangle}

\def\Rad{\mathfrak{R}}
\def\hcR{\widehat{\cR}}
\def\mop{\cM}
\def\conv{\textup{conv}}

\numberwithin{equation}{section}
\declaretheorem[numberwithin=section]{theorem}
\declaretheorem[numberlike=theorem]{lemma}

\declaretheoremstyle[qed={\ensuremath\Diamond}]{remstyle}

\usepackage{mathtools}
\usepackage{dylan}

\def\srelu{\sigma_{\textup{r}}}

\def\cifar{\texttt{cifar}\xspace}
\def\ciften{\texttt{cifar10}\xspace}
\def\cifhun{\texttt{cifar100}\xspace}
\def\mnist{\texttt{mnist}\xspace}
\def\polylog{\textup{poly\,log}}

\title{Spectrally-normalized margin bounds for neural networks}

\author{
  Peter L. Bartlett\thanks{{\tt<\url{peter@berkeley.edu}>}; University
  of California, Berkeley and Queensland University of Technology; work performed while visiting the Simons Institute.}
  \and
  Dylan J. Foster\thanks{{\tt<\url{djf244@cornell.edu}>}; Cornell University; work performed while visiting the Simons Institute.}
  \and
  Matus Telgarsky\thanks{{\tt<\url{mjt@illinois.edu}>}; University of Illinois, Urbana-Champaign; work performed while visiting the Simons Institute.}
}

\date{} 
\begin{document}

\maketitle

\begin{abstract}
  This paper presents a margin-based multiclass generalization bound for neural networks
  that scales with their margin-normalized \emph{spectral complexity}: their Lipschitz constant,
  meaning the product of the spectral norms of the weight matrices, times a certain correction factor.
  This bound is empirically investigated for a standard AlexNet network trained
  with SGD on the \mnist and \ciften datasets, with both original and random
  labels;  the bound, the Lipschitz constants, and the excess risks are all in direct correlation,
  suggesting both that SGD selects predictors whose complexity scales with the difficulty of the learning task,
  and secondly that the presented bound is sensitive to this complexity.
\end{abstract}

\section{Overview}
\label{sec:intro}

Neural networks owe their astonishing success not only to their ability
to fit any data set:
they also \emph{generalize well}, meaning they provide a close fit on unseen data.
A classical statistical adage is that models capable of fitting too much will
generalize poorly; what's going on here?

Let's navigate the many possible explanations provided by statistical theory.
A first observation is that any analysis based solely on the number of possible labellings on a finite
training set --- as is the case with VC dimension --- is doomed:
if the function class can fit all possible labels (as is the case with neural networks in standard configurations \citep{rethinking}),
then this analysis can not distinguish it from the collection of all possible functions!

\begin{figure}[h!]
  \centering
  \includegraphics[width=0.71\textwidth]{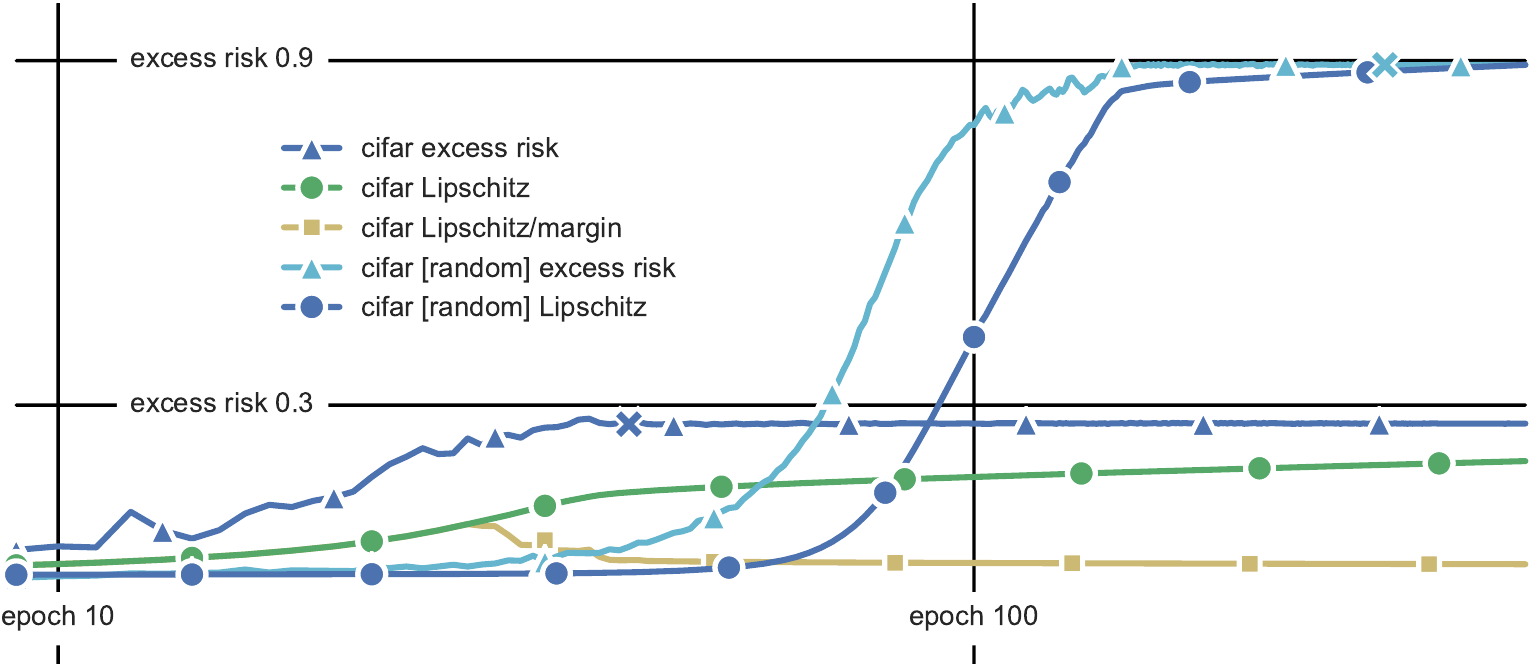}
  \caption{
    An analysis of AlexNet \citep{imagenet_sutskever} trained with SGD on \ciften, both with original and with random labels.
    Triangle-marked curves track excess risk across training epochs (on a log scale),
    with an `{\tt x}' marking the earliest epoch with zero training error.
    Circle-marked curves track Lipschitz constants,
    normalized so that the two curves for random labels meet.
    The Lipschitz constants tightly correlate with excess risk,
    and moreover normalizing them by \emph{margins} (resulting in the square-marked curve) neutralizes growth across epochs.
  \label{fig:intro:excess}}
\end{figure}

Next let's consider \emph{scale-sensitive} measures of complexity,
such as Rademacher complexity and covering numbers,
which (can) work directly with real-valued function classes,
and moreover are sensitive to their magnitudes.
\Cref{fig:intro:excess} plots the excess risk (the test error minus the training error)
across training epochs against one candidate scale-sensitive complexity measure, the Lipschitz constant of the network
(the product of the spectral norms of the weight matrices),
and demonstrates that they are
tightly correlated (which is not the case for, say, the $l_2$ norm of the weights).
The data considered in \Cref{fig:intro:excess} is the standard \ciften dataset, both with
original and with random labels, which has been used as a sanity check
when investigating
neural network generalization \citep{rethinking}.

There is still an issue with basing a complexity measure purely on the Lipschitz constant
(although it has already been successfully employed to regularize neural networks \citep{parseval_networks}):
as depicted in \Cref{fig:intro:excess}, the measure grows over time, despite the excess risk plateauing.
Fortunately, there is a standard resolution to this issue: investigating the \emph{margins}
(a precise measure of confidence) of the outputs of the network.
This tool has been used to study the behavior of 2-layer networks, boosting methods, SVMs,
and many others
\citep{bartlett_margin,boosting_margin,esaim_survey};
in boosting, for instance, there is a similar growth in complexity over time (each training iteration adds a weak learner),
whereas margin bounds correctly stay flat or even decrease.
This behavior is recovered here:
as depicted in \Cref{fig:intro:excess},
even though standard networks exhibit growing Lipschitz constants,
normalizing these Lipschitz constants by the margin instead gives a decaying curve.

\subsection{Contributions}

This work investigates a complexity measure for neural networks that is based on the Lipschitz constant,
but normalized by the margin of the predictor.
The two central contributions are as follows.
\begin{itemize}
  \item
    \Cref{fact:main:new} below will give the rigorous statement of the
    generalization bound that is the basis
    of this work.
    In contrast to prior work, this bound:
    \textbf{(a)} scales with the Lipschitz constant (product of spectral norms of weight matrices) divided by the margin;
    \textbf{(b)} has no dependence on combinatorial parameters (e.g., number of layers or nodes) outside of log factors;
    \textbf{(c)} is multiclass (with no explicit dependence on the number of classes);
    \textbf{(d)} measures complexity against a \emph{reference network} (e.g., for the ResNet \citep{resnet},
    the reference network has identity mappings at each layer).
    The bound is stated below, with a general form and analysis summary appearing in \Cref{sec:theory}
    and the full details relegated to the appendix.

  \item
    An empirical investigation, in \Cref{sec:empirical}, of neural network generalization on the standard datasets \ciften, \cifhun, and \mnist
    using the preceding bound.
    Rather than using the bound to provide a single number,
    it can be used to form a \emph{margin distribution}
    as in \Cref{fig:intro:margin_dists}.
    These margin distributions will illuminate the following intuitive observations:
    \textbf{(a)} \ciften is harder than \mnist;
    \textbf{(b)} random labels make \ciften and \mnist much more difficult;
    \textbf{(c)} the margin distributions (and bounds) converge during training, even though the weight matrices continue to grow;
    \textbf{(d)} $l_2$ regularization (``weight decay'') does not significantly impact margins or generalization.
\end{itemize}
A more detailed description of the margin distributions is as follows.  Suppose a neural
network computes a function $f:\R^d \to \R^k$, where $k$ is the number of classes;
the most natural way to convert this to a classifier is
to select the output coordinate with the largest magnitude, meaning $x \mapsto \argmax_j f(x)_j$.
The \emph{margin}, then, measures the gap between the output for the correct label and
other labels, meaning
$f(x)_y - \max_{j\neq y} f(x)_j$.

Unfortunately, margins alone do not seem to say much;
see for instance \Cref{fig:intro:margin_dists:unnorm}, where the collections
of all margins for all data points --- the \emph{unnormalized margin distribution} ---
are similar for \ciften with and without random labels.
What is missing is an appropriate \emph{normalization},
as in \Cref{fig:intro:margin_dists:norm}.
This normalization is provided by \Cref{fact:main:new},
which can now be explained in detail.

\begin{figure}[t]
  \centering
  \begin{subfigure}[b]{0.48\textwidth}
  \includegraphics[width = 1.0\textwidth]{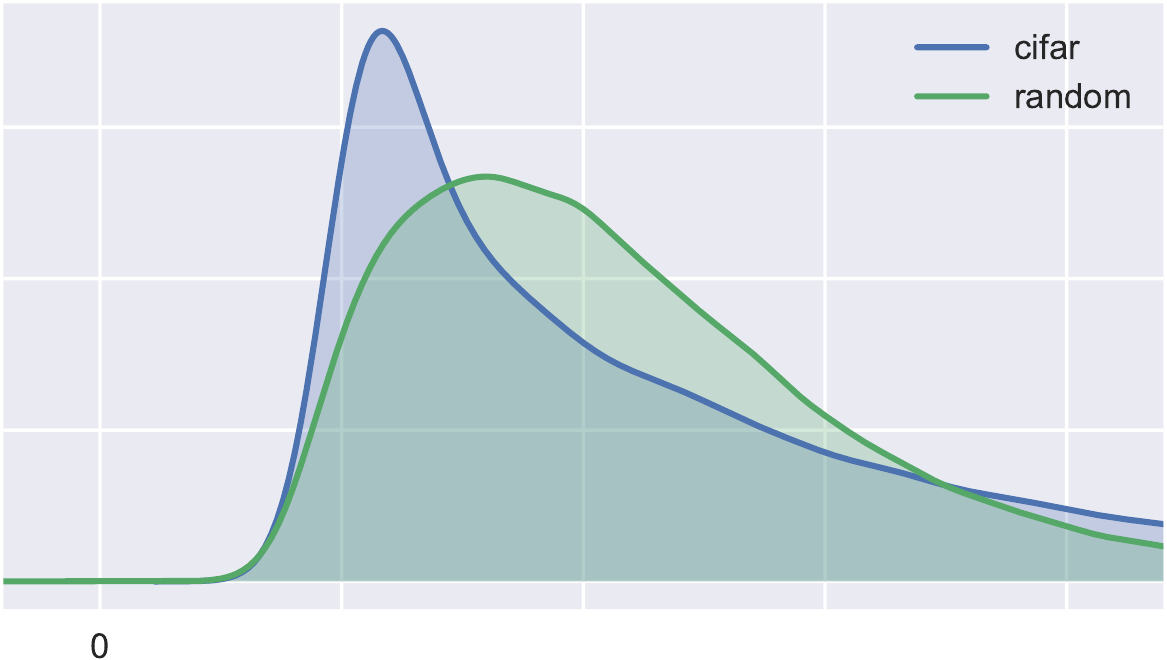}
    \caption{Margins.\label{fig:intro:margin_dists:unnorm}}
  \end{subfigure}
  \begin{subfigure}[b]{0.48\textwidth}
    \includegraphics[width = 1.0\textwidth]{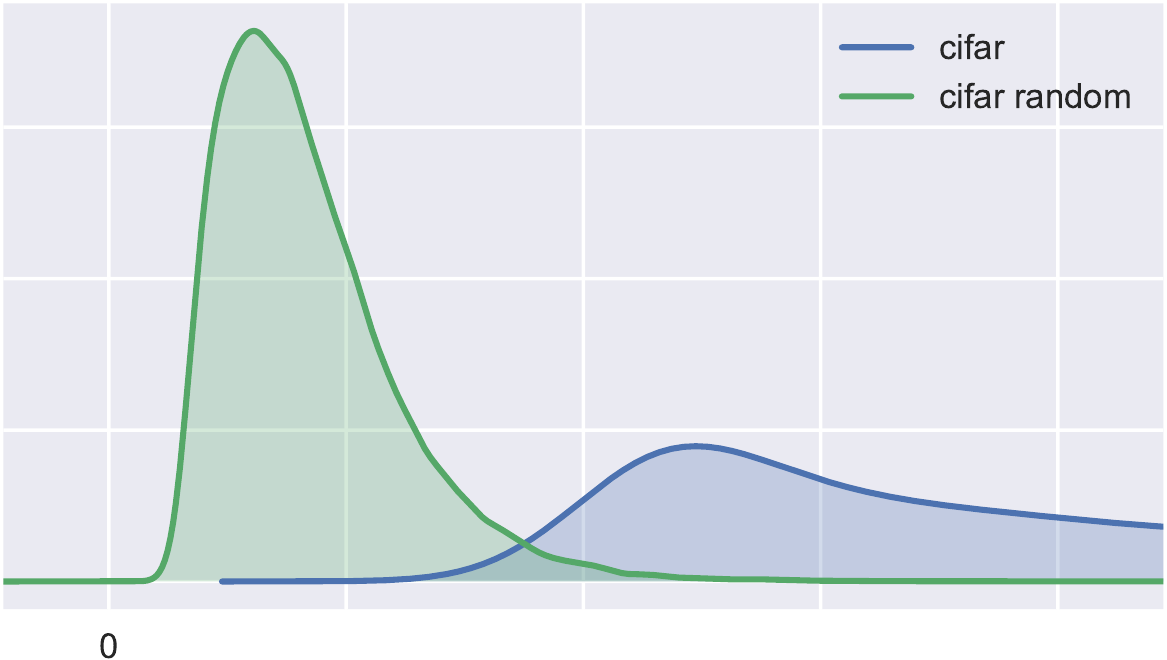}
    \caption{Normalized margins.\label{fig:intro:margin_dists:norm}}
  \end{subfigure}
  \caption{Margin distributions at the end of training AlexNet on \ciften,
  with and without random labels.
  With proper normalization, random labels demonstrably correspond to a harder
    problem.\label{fig:intro:margin_dists}}
\end{figure}

To state the bound, a little bit of notation is necessary.
The networks will use $L$ fixed nonlinearities $(\sigma_1,\ldots, \sigma_L)$,
where $\sigma_i:\R^{d_{i-1}}\to\R^{d_i}$ is $\rho_i$-Lipschitz
(e.g., as with coordinate-wise ReLU, and max-pooling, as discussed in \Cref{sec:app:prelim});
occasionally, it will also hold that $\sigma_i(0) = 0$.
Given $L$ weight matrices $\cA = (A_1,\ldots,A_L)$
let $F_\cA$ denote the function computed by the corresponding network:
\begin{equation}\label{eq:FA}
  F_\cA(x) := \sigma_L(A_L\sigma_{L-1}(A_{L-1} \cdots \sigma_1(A_1 x)\cdots)).
\end{equation}
The network output $F_\cA(x)\in\R^{d_L}$ (with $d_0=d$ and $d_L=k$)
is converted
to a class label in $\{1,\ldots,k\}$ by taking the $\arg\max$ over
components, with an arbitrary rule for breaking ties.
Whenever input data $x_1,\ldots,x_n\in\R^d$ are given,
collect them as rows of a matrix $X \in \R^{n\times d}$.
Occasionally, notation will be overloaded to discuss $F_\cA(X^T)$, a matrix whose $i^{\textup{th}}$
column is $F_\cA(x_i)$.
Let $W$ denote the maximum of $\{d,d_1,\ldots,d_L\}$.
The $l_2$ norm $\|\cdot\|_2$ is always computed entry-wise; thus,
for a matrix, it corresponds to the Frobenius norm.

Next, define a collection of \emph{reference matrices}
$(M_1,\ldots,M_L)$ with the same dimensions as $A_1,\ldots,A_L$; for instance,
to obtain a good bound for ResNet \citep{resnet}, it is sensible to set $M_i := I$, the identity map,
and the bound below will worsen as the network moves farther from the identity map;
for AlexNet \citep{imagenet_sutskever}, the simple choice $M_i=0$ suffices.
Finally,
let $\|\cdot\|_\sigma$ denote the spectral norm,
and let $\|\cdot\|_{p,q}$ denote the $(p,q)$ matrix norm, defined by
        $\nrm*{A}_{p,q} := \enVert{ (\|{}A_{:,1}\|_p, \ldots, \|{}A_{:,m} \|_p)}_q$ for $A\in\bbR^{d\times{}m}$.
The \emph{spectral complexity} $R_{F_\cA} = R_\cA$ of a network $F_\cA$ with weights $\cA$ is the defined as
\begin{equation}
  R_{\cA}
  :=
  \del{\prod_{i=1}^L \rho_i \|A_i\|_\sigma}
  \del{\sum_{i=1}^L \frac{ \|A_i^{\top} - M_i^{\top}\|_{2,1}^{2/3}}{\|A_i\|_\sigma^{2/3}}}^{3/2}.
  \label{eq:spec_comp}
\end{equation}

The following theorem provides a generalization bound for neural networks whose nonlinearities are
fixed but whose weight matrices $\cA$ have bounded spectral complexity $R_\cA$.

\begin{theorem}
  \label{fact:main:new}
  Let nonlinearities $(\sigma_1,\ldots,\sigma_L)$
  and
  reference matrices $(M_1,\ldots, M_L)$
  be given as above (i.e., $\sigma_i$ is $\rho_i$-Lipschitz and $\sigma_i(0) = 0$).
  Then for $(x,y),(x_1,y_1),\ldots,(x_n,y_n)$ drawn iid from any
  probability distribution over $\R^d\times\{1,\ldots,k\}$,
  with probability at least $1-\delta$ over $((x_i,y_i))_{i=1}^n$,
  every margin $\gamma > 0$
  and network $F_\cA : \R^d \to \R^k$ with weight matrices $\cA = (A_1,\ldots,A_L)$
  satisfy
  \begin{align*}
    \Pr\sbr[2]{ \argmax_j F_\cA(x)_j \neq y }
    \leq
    \hcR_\gamma(F_\cA)
    + \widetilde{\cO}
    \del{ \frac {\|X\|_2 R_\cA}{\gamma n}    \ln(W) + \sqrt{\frac {\ln(1/\delta)}{n}} },
   \end{align*}
  where $\hcR_\gamma(f) \leq n^{-1} \sum_i \1\sbr{ f(x_i)_{y_i} \leq \gamma + \max_{j\neq y_i} f(x_i)_j }$ and $\|X\|_2 = \sqrt{ \sum_i \|x_i\|_2^2 }$.
\end{theorem}

The full proof and a generalization beyond spectral norms is relegated to the appendix,
but a sketch is provided in \Cref{sec:theory},
along with a lower bound.
\Cref{sec:theory} also gives a discussion of related work:
briefly, it's essential
to note that margin and Lipschitz-sensitive bounds have a long history in the neural networks literature
\citep{bartlett_margin,anthony_bartlett_nn,neyshabur2015norm};
the distinction here is the sensitivity to the spectral norm,
and that there is no explicit appearance of combinatorial quantities such as numbers of parameters
or layers (outside of log terms, and indices to summations and products).

To close, miscellaneous observations and open problems are collected in \Cref{sec:open}.

\section{Generalization case studies via margin distributions}
\label{sec:empirical}

In this section, we empirically study the generalization behavior of
neural networks, via margin distributions and the generalization bound
stated in \Cref{fact:main:new}.

Before proceeding with the plots, it's a good time
to give a more refined description of the margin distribution, one that is
suitable for comparisons across datasets.
Given $n$ pattern/label pairs $((x_i,y_i))_{i=1}^n$, with patterns as
rows of matrix $X\in\R^{n\times d}$, and given a predictor
$F_\cA:\R^d\to\R^k$,
the (normalized) margin distribution is the univariate empirical
distribution of the labeled data points each transformed into a single scalar according to
\[
  (x,y) \mapsto \frac {F_\cA(x)_y - \max_{i\neq y}F_\cA(x)_i}{R_\cA \|X\|_{2}/n},
\]
where the spectral complexity $R_\cA$ is from \cref{eq:spec_comp}.
The normalization is thus derived from the bound in \Cref{fact:main:new},
but ignoring log terms.

\begin{figure}
  \centering
  \begin{subfigure}[t]{0.49\textwidth}
    \includegraphics[width = 1.0\textwidth]{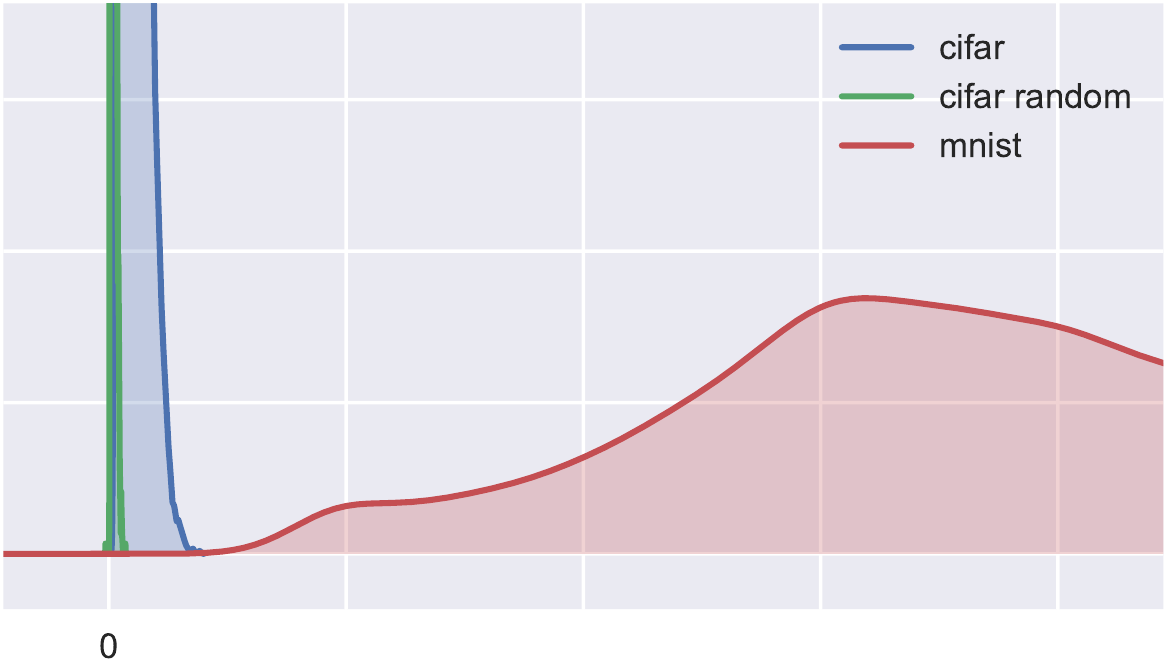}
    \caption{\mnist is easier than \ciften.\label{fig:mnist:1}}
  \end{subfigure}\hfill
  \begin{subfigure}[t]{0.49\textwidth}
    \includegraphics[width = 1.0\textwidth]{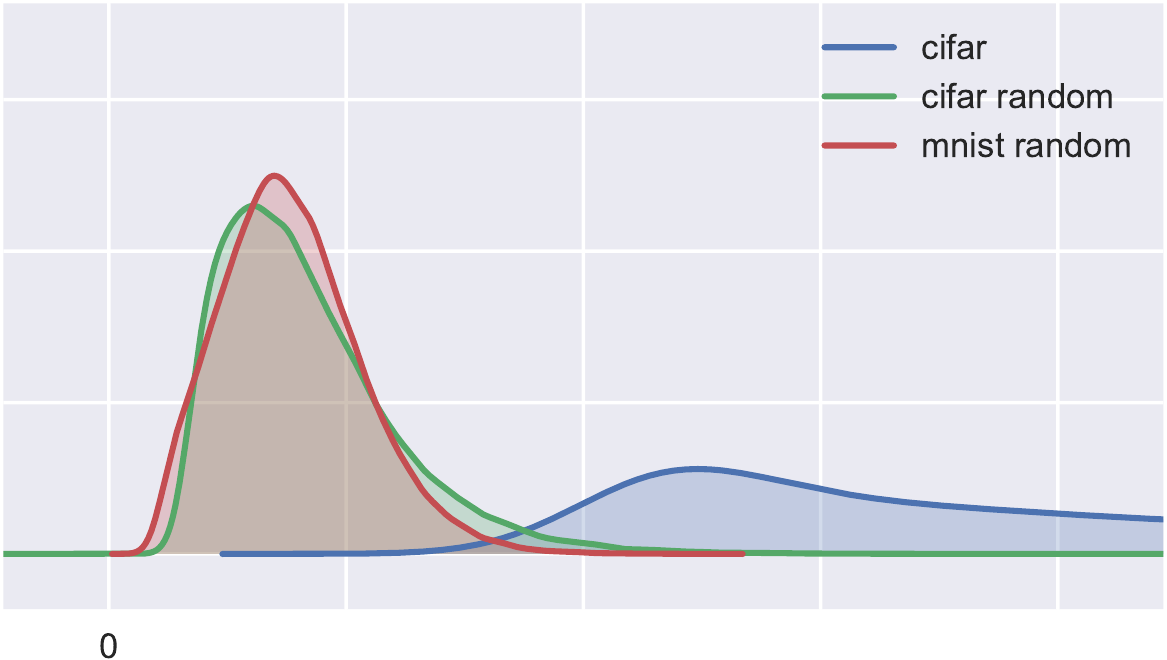}
    \caption{Random \mnist is as hard as random \ciften!\label{fig:mnist:2}}
  \end{subfigure}
  \vfill
  \begin{subfigure}[t]{0.49\textwidth}
    \includegraphics[width = 1.0\textwidth]{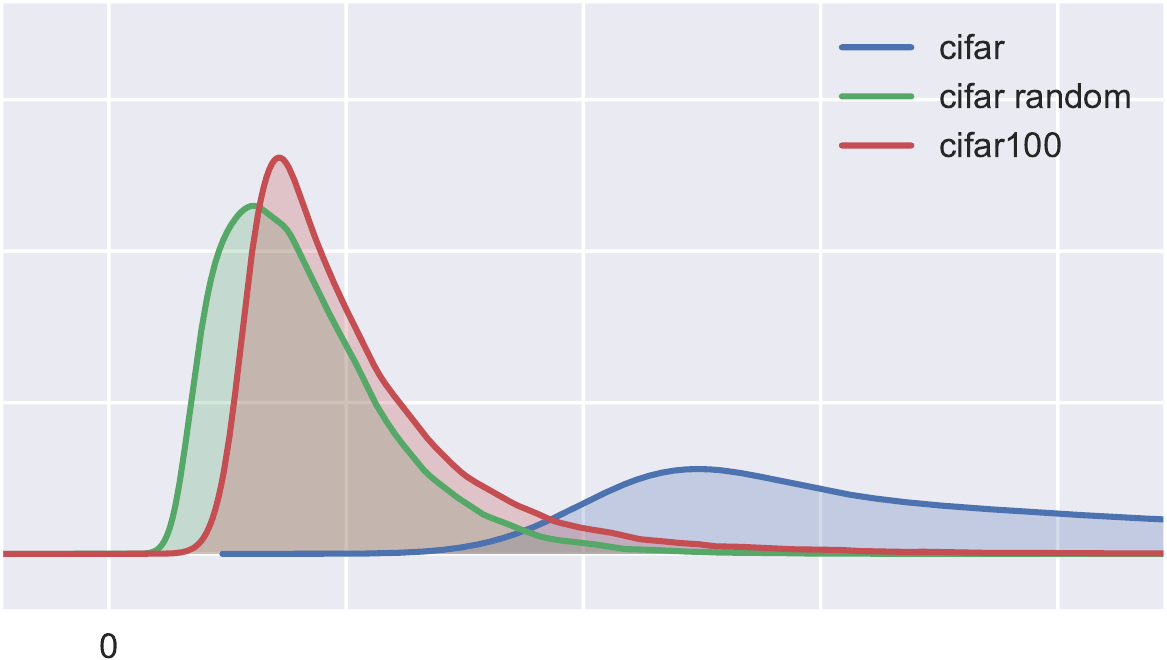}
    \caption{\cifhun is as hard as \ciften
    with random labels! \label{fig:cifar100}    }\end{subfigure}\hfill
    \begin{subfigure}[t]{0.49\textwidth}
    \includegraphics[width = 1.0\textwidth]{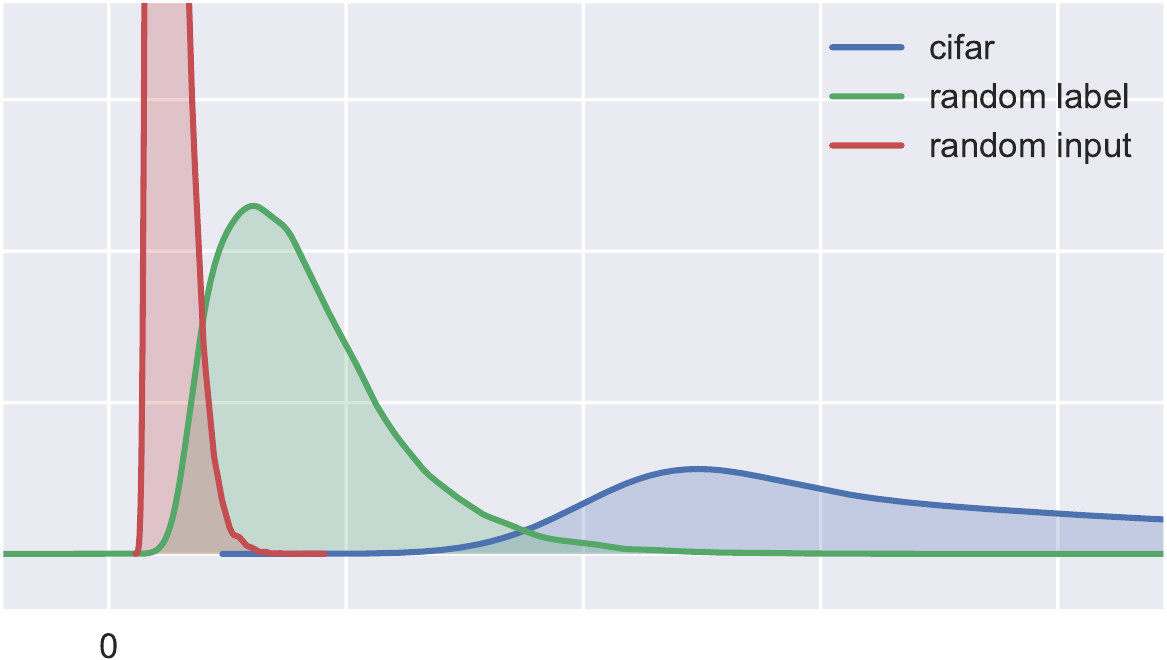}
    \caption{Random inputs are harder than random labels.\label{fig:randim}}
  \end{subfigure}
  \caption{A variety of margin distributions.  Axes are re-scaled in \Cref{fig:mnist:1},
  but identical in the other subplots;
  the \ciften (blue) and random \ciften (green) distributions are the same each time.}
\end{figure}

Taken this way, the two margin distributions for two datasets can
be interpreted as follows.  Considering any fixed point on the
horizontal axis, if the \emph{cumulative} distribution of one density
is lower than the other, then it corresponds to a lower right hand
side in \Cref{fact:main:new}.
For no reason other than visual interpretability,
the plots here will instead depict a density estimate of the margin distribution.
The vertical and horizontal axes are rescaled in different plots, but the random and
true \ciften margin distributions are always the same.

A little more detail about the experimental setup is as follows.
All experiments were implemented in Keras \citep{chollet2015keras}. In order to minimize conflating effects of optimization and regularization,
the optimization method was vanilla SGD with step size $0.01$,
and all regularization (weight decay, batch normalization, etc.)
were disabled.  ``\cifar'' in general refers to \ciften, however \cifhun will
also be explicitly mentioned.  The network architecture is essentially AlexNet
\citep{imagenet_sutskever}
with all normalization/regularization removed, and with
no adjustments of any kind (even to the learning rate) across the different experiments.

\textbf{Comparing datasets.}\quad
A first comparison is of \ciften and the standard \mnist digit data.  \mnist is
considered ``easy'', since any of a variety of
methods can achieve roughly 1\% test error.
The ``easiness'' is corroborated by \Cref{fig:mnist:1}, where the margin
distribution for \mnist places all its mass far to the right of the mass for
\ciften.  Interestingly, randomizing the labels of \mnist, as in \Cref{fig:mnist:2},
results in a margin distribution to the left of not only \ciften, but also slightly to the left of
(but close to) \ciften with randomized labels.

Next, \Cref{fig:cifar100} compares \ciften and \cifhun, where \cifhun uses the same input images as
\ciften; indeed, \ciften is obtained from \cifhun by collapsing the original 100
categories into 10 groups.  Interestingly, \cifhun, from the perspective of
margin bounds, is just as difficult as \ciften with random labels.  This is
consistent with the large observed test error on \cifhun (which has not been ``optimized'' in any
way via regularization).

Lastly, \Cref{fig:randim} replaces the \ciften \emph{input images} with random images
  sampled from Gaussians matching the first- and second-order image statistics
  (see \citep{rethinking} for similar experiments). 

\begin{figure}
  \centering
  \begin{subfigure}[b]{0.49\textwidth}
    \includegraphics[width = 1.0\textwidth]{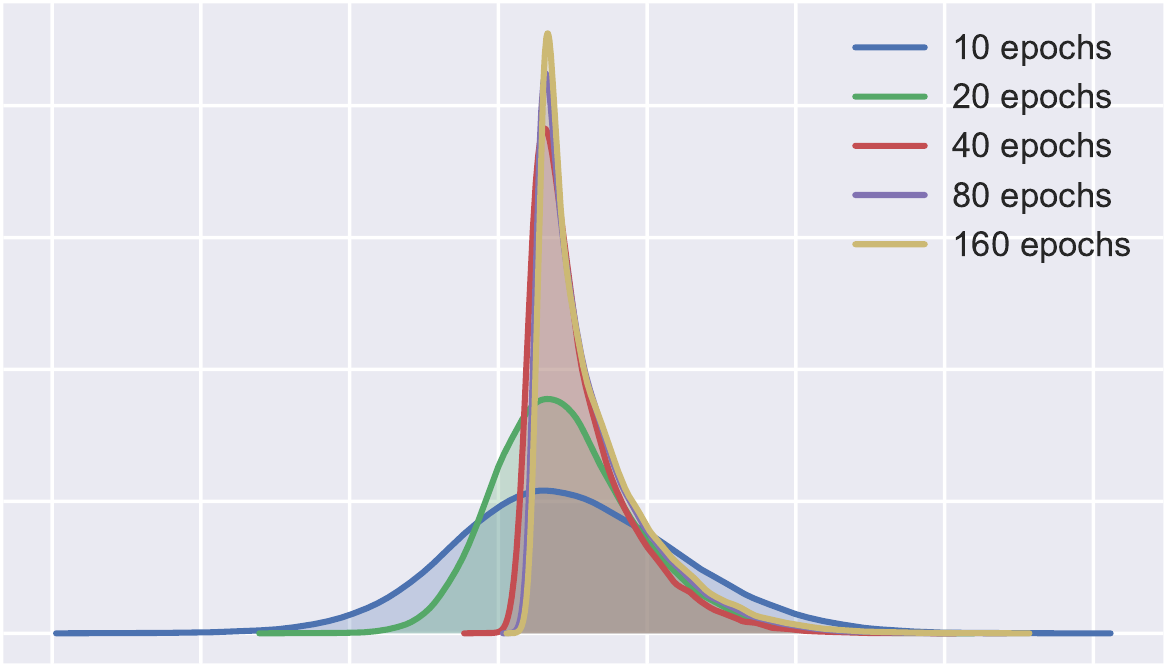}
    \caption{Margins across epochs for \ciften. \label{fig:margins_converge}}
  \end{subfigure}\hfill
  \begin{subfigure}[b]{0.49\textwidth}
    \includegraphics[width = 1.0\textwidth]{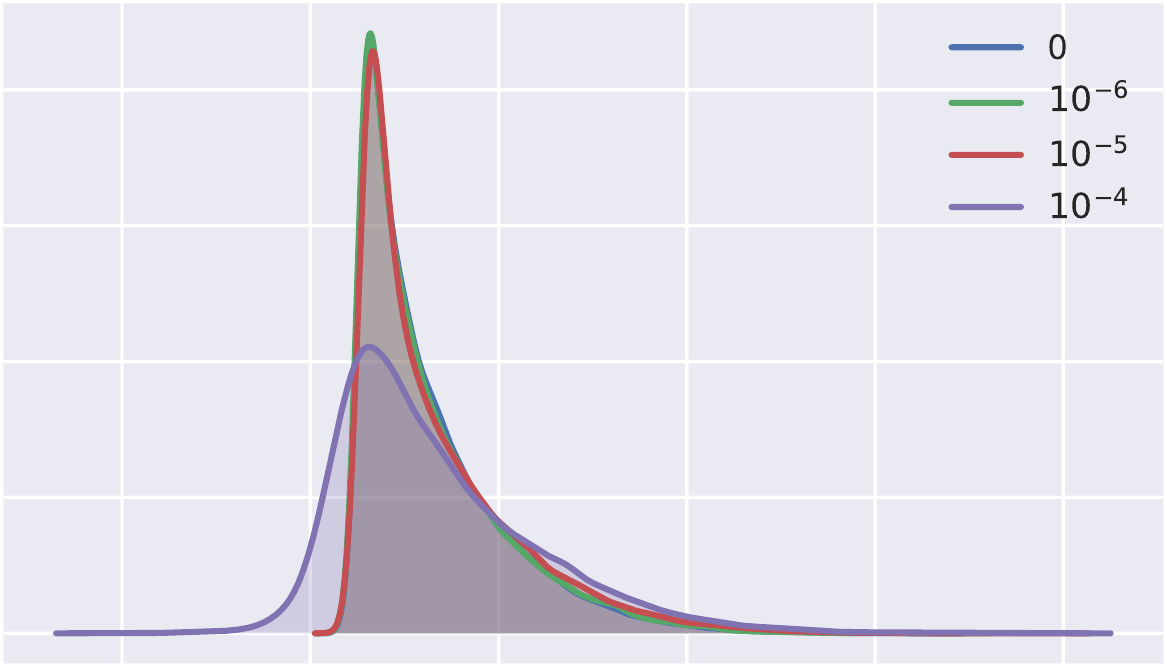}
    \caption{Various levels of $l_2$ regularization for \ciften. \label{fig:l2}}
  \end{subfigure}
  \caption{}
\end{figure}

\textbf{Convergence of margins.}\quad
As was pointed out in \Cref{sec:intro}, the weights of the neural networks
do not seem to converge in the usual sense during training (the norms grow continually).  However,
as depicted in \Cref{fig:margins_converge}, the sequence of (normalized)
margin distributions
is itself converging.

\textbf{Regularization.}\quad
As remarked in \citep{rethinking}, regularization only seems to bring minor
benefits to test error (though adequate to be employed in all cutting edge results).
  This observation is certainly consistent with the margin distributions in \Cref{fig:l2}, which
do not improve (e.g., by shifting to the right) in any visible way under regularization.  An open question,
discussed further in \Cref{sec:open}, is to design regularization that
improves margins.

\section{Analysis of margin bound}
\label{sec:theory}

This section will sketch the proof of \Cref{fact:main:new},
give a lower bound,
and discuss related work.

\subsection{Multiclass margin bound}
\label{sec:multiclass_margins}

The starting point of this analysis is a margin-based bound for multiclass prediction.
To state the bound, first recall that the \emph{margin operator}
$\mop:\R^k\times\{1,\ldots,k\}\to\R$ is defined as
$\displaystyle
  \mop(v,y) := v_y - \max_{i\neq y} v_i
$,
and define the \emph{ramp loss} $\ell_\gamma:\R\to\R^+$ as
\[
  \ell_\gamma(r) :=
  \begin{cases}
    0
    &r< -\gamma,
    \\
    1 + r/\gamma
    &r \in [-\gamma,0],
    \\
    1
    &
    r > 0,
  \end{cases}
\]
and \emph{ramp risk} as $\displaystyle \cR_\gamma(f) := \bbE(\ell_\gamma(-\mop(f(x),y)))$.
Given a sample $S:=((x_1,y_1),\ldots,(x_n,y_n))$,
define an empirical counterpart $\hcR_\gamma$ of $\cR_\gamma$ as
$\hcR_\gamma(f) := n^{-1} \sum_i \ell_\gamma(-\mop(f(x_i), y_i))$;
note that $\cR_\gamma$ and $\hcR_\gamma$ respectively upper bound the
probability and fraction of errors
on the source distribution and training set.
Lastly, given a set of real-valued functions $\cH$,
define the \emph{Rademacher complexity} as $\Rad(\cH_{|S}) := n^{-1}\bbE
\sup_{h \in \cH} \sum_{i=1}^n\eps_i h(x_i,y_i)$,
where the expectation is over the Rademacher random variables
$(\eps_1,\ldots,\eps_n)$,
which are iid with $\Pr[\eps_1 = 1] = \Pr[\eps_1 = -1] = \nicefrac 1 2$.

With this notation in place, the basic bound is as follows.

\begin{lemma}
  \label{fact:margin_multiclass}
   Given functions $\cF$ with $\cF\ni f : \R^d \to \R^k$ and any $\gamma>0$,
   define
   \[
     \cF_\gamma :=
     \cbr{
       (x,y) \mapsto \ell_\gamma(-\mop(f(x), y))
       :
       f\in \cF
     }.
   \]
   Then, with probability at least $1-\delta$
   over a sample $S$ of size $n$,
   every $f\in \cF$ satisfies
     \[
     \Pr[ \argmax_i f(x)_i \neq y ]
     \leq
     \hcR_\gamma(f)
     +
     2 \Rad((\cF_\gamma)_{|S})
     + 3 \sqrt{\frac{\ln(1/\delta)}{2n}}.
   \]
   \end{lemma}

This bound is a direct consequence of standard tools in Rademacher complexity.
In order to instantiate this bound, covering numbers will be used to directly
upper bound the Rademacher complexity term $\Rad((\cF_\gamma)_{|S})$.
Interestingly, the choice of directly working in terms of covering numbers
seems essential to providing a bound with no explicit dependence on $k$;
by contrast, prior work primarily handles multiclass via a Rademacher
complexity analysis on each coordinate of a $k$-tuple of functions,
and pays a factor of $\sqrt{k}$
\citep{tz_multiclass_consistency}.

\subsection{Covering number complexity upper bounds}

This subsection proves
\Cref{fact:main:new} via \Cref{fact:margin_multiclass}
by controlling, via covering numbers,
the Rademacher complexity $\Rad((\cF_\gamma)_{|S})$
for networks with bounded spectral complexity.

The notation here for (proper) covering numbers is as follows.
Let $\cN(U, \eps, \|\cdot\|)$ denote the least cardinality of any
subset $V\subseteq U$
that \emph{covers} $U$ at scale $\eps$ with norm $\|\cdot\|$, meaning
\[
  \sup_{A\in U} \min_{B\in V} \|A-B\| \leq \eps.
\]
Choices of $U$ that will be used in the present work include
both the image $\cF_{|S}$ of data $S$ under some function class $\cF$,
as well as the conceptually simpler choice of a family of matrix
products.

The full proof has the following steps.
\textbf{(I)} A \emph{matrix covering} bound for the affine transformation
of each layer is provided in \Cref{fact:matrix_l21_covering};
handling whole layers at once allows for more flexible norms.
\textbf{(II)} An induction on layers then gives a covering number
bound for entire networks; this analysis is only sketched here for the
special case of norms used in \Cref{fact:main:new}, but the full proof
in the appendix culminates in a bound for more general norms
(cf. \Cref{fact:cover:general}).
\textbf{(III)} The preceding whole-network covering number
leads to \Cref{fact:main:new} via \Cref{fact:margin_multiclass}
and standard techniques.

Step \textbf{(I)}, \emph{matrix covering}, is handled by the following
lemma.  The covering number considers the matrix product $XA$, where $A$
will be instantiated as the weight matrix for a layer, and $X$ is the
data passed through all layers prior to the present layer.

\begin{lemma}
  \label{fact:matrix_l21_covering}
  Let conjugate exponents $(p,q)$ and $(r,s)$ be given with
  $p \leq 2$,
  as well as positive reals $(a,b,\eps)$ and positive integer $m$.
  Let matrix $X \in \R^{n\times d}$ be given with $\|X\|_{p} \leq b$.
  Then
    \[
    \ln \cN\del{\cbr{XA : A\in \R^{d\times m}, \|A\|_{q,s}\leq a}, \eps, \|\cdot\|_2}
    \leq
    \left\lceil \frac{a^2 b^2 m^{2/r}}{ \eps^2}\right \rceil \ln(2dm).
  \]
  \end{lemma}
The proof relies upon the
\emph{Maurey sparsification lemma} \citep{pisier1980remarques},
which is stated in terms of sparsifying convex hulls,
and in its use here is inspired by covering number bounds for linear predictors \citep{tz_covering}.
To prove \Cref{fact:main:new},
this matrix covering bound
will be instantiated for the case of $\nrm*{A}_{2,1}$.
It is possible to instead scale with $\|A\|_2$ and $\|X\|_2$,
but even for the case of the identity matrix $X=I$,
this incurs an extra dimension factor.
The use of $\|A\|_{2,1}$ here thus helps \Cref{fact:main:new}
avoid any appearance of $W$ and $L$ outside of log terms;
indeed, the goal of covering a whole matrix at a time (rather than
the more standard vector covering) was to allow this greater
sensitivity and avoid combinatorial parameters.

Step \textbf{(II)}, the induction on layers, proceeds as follows.
Let $X_i$ denote the output of layer $i$ but with images of examples of columns (thus $X_0 = X^\top$),
and inductively suppose there exists a cover element $\widehat X_i$ for $X_i$
which depends on covering matrices $(\widehat A_1,\ldots,\widehat A_{i-1})$
chosen to cover weight matrices in earlier layers.
Thanks to \Cref{fact:matrix_l21_covering},
there also exists $\widehat A_i$
so that $\|A_i \widehat X_i - \widehat A_i \widehat X_i \|_2 \leq \eps_i$.
The desired cover element is thus
$\widehat X_{i+1} = \sigma_i(\widehat A_i \widehat X_i)$ where $\sigma_i$ is the
nonlinearity in layer $i$;
indeed, supposing $\sigma_i$ is $\rho_i$-Lipschitz,
  \begin{align*}
    \|X_{i+1} - \widehat X_{i+1}\|_2
    &\leq \rho_i \|A_iX_i - \widehat A_i \widehat X_i\|_2
    \\
    &\leq \rho_i
    \del{ \|A_iX_i - A_i\widehat X_i\|_2 + \|A_i\widehat X_i - \widehat A_i \widehat X_i\|_2}
    \\
    &
    \leq
    \rho_i \|A_i\|_\sigma
    \|X_i - \widehat X_i\|_2
    + \rho_i \eps_i,
  \end{align*}
  where the first term is controlled with the inductive hypothesis.
  Since $\widehat X_{i+1}$
  depends on each choice $(\widehat A_i,\ldots,\widehat A_i)$,
  the cardinality of the full network cover is the product of
  the individual matrix covers.

  The preceding proof had no sensitivity to the particular choice
  of norms; it merely required an operator norm on $A_i$,
  as well as some other norm that allows matrix covering.
  Such an analysis is presented in full generality in
  \Cref{sec:cover:general}.
  Specializing to the particular case of spectral norms and
  $(2,1)$ group norms leads to the following full-network
  covering bound.

\begin{theorem}
  \label{fact:cover:spectral}
  Let fixed nonlinearities $(\sigma_1,\ldots,\sigma_L)$
  and
  reference matrices $(M_1,\ldots, M_L)$
  be given, where $\sigma_i$ is $\rho_i$-Lipschitz and $\sigma_i(0)=0$.
  Let spectral norm bounds $(s_1,\ldots,s_L)$,
  and matrix $(2,1)$ norm bounds $(b_1,\ldots,b_L)$ be given.
  Let data matrix $X\in\R^{n\times d}$ be given,
  where the $n$ rows correspond to data points.
  Let $\cH_X$ denote the family of matrices obtained by
  evaluating $X$ with all choices of network $F_\cA$:
  \[
    \cH_X := \cbr{
      F_\cA(X^T)\ :\ \cA = (A_1,\ldots,A_L),
      \ \|A_i\|_\sigma \leq s_i,
      \ \|A_i^{\top} - M_i^{\top}\|_{2,1} \leq b_i
    },
  \]
  where each matrix has dimension at most $W$ along each axis.
  Then for any $\eps>0$,
     \[
    \ln \cN(\cH_X, \eps, \|\cdot\|_2)
    \leq
    \frac {\|X\|^2_2 \ln(2W^2)}{\eps^2}
    \del{
      \prod_{j=1}^L s_j^2\rho_j^2
    }
    \del{
      \sum_{i=1}^L \del{\frac {b_i}{s_i}}^{2/3}
    }^3.
  \]
 \end{theorem}

What remains is \textbf{(III)}:
\Cref{fact:cover:spectral} can be combined with the standard Dudley entropy
integral upper bound on Rademacher complexity (see e.g. \cite{mohri_book}), which together with
\Cref{fact:margin_multiclass} gives \Cref{fact:main:new}.

\subsection{Rademacher complexity lower bounds}

By reduction to the linear case (i.e., removing all nonlinearities),
it is easy to provide a lower bound on the Rademacher complexity of the
networks studied here.  Unfortunately, this bound only scales with the product of spectral
norms, and not the other terms in $R_\cA$ (cf. \cref{eq:spec_comp}).

\begin{theorem}
  \label{fact:spectral_lb}
  Consider the setting of \Cref{fact:cover:spectral},
  but all nonlinearities are the ReLU $z\mapsto \max\{0,z\}$,
  the output dimension is $d_L=1$,
  and all non-output dimensions are at least 2 (and hence $W\geq 2$).
  Let data $S:= (x_1,\ldots,x_n)$ be collected into data matrix $X \in \R^{n\times d}$.
  Then there is a $c$ such that for any scalar $r>0$,
  \begin{equation}
    \Rad\del[3]{
      \cbr[2]{ F_\cA : \cA = (A_1,\ldots,A_L),\ \prod_i \|A_i\|_\sigma \leq r }_{|S}
    } \geq c\nrm*{X}_{2}r.
  \end{equation}
\end{theorem}

Note that, due to the nonlinearity, the lower bound should indeed depend on $\prod_i \|A_i\|_\sigma$ and not $\|\prod_i A_i\|_\sigma$;
as a simple sanity check, there exist networks for which the latter quantity is 0, but the network does
not compute the zero function.

\subsection{Related work}

To close this section on proofs, it is a good time to summarize connections to existing literature.

The algorithmic idea of large margin classifiers was introduced in the
linear case by \citet{VapnikBook1982}
(see also~\citep{BoserGuyonVapnik1992,CortesVapnik1995}).
\citet{VapnikBook1995} gave an intuitive explanation of the
performance of these methods based on a sample-dependent
VC-dimension calculation, but without generalization bounds.
The first rigorous generalization bounds for large margin linear
classifiers~\citep{STBWA1998} required a
scale-sensitive complexity analysis of real-valued function classes.
At the same time, a large margin analysis was developed for two-layer
networks \citep{bartlett_margin}, indeed with
a proof technique that inspired the layer-wise induction used to prove \Cref{fact:main:new} in the present work.
Margin theory was quickly extended to many other settings (see for instance the survey by \citet{esaim_survey}),
one major success being
an explanation of the generalization ability of boosting methods, which exhibit an explicit growth
in the size of the function class over time, but a stable excess risk \citep{boosting_margin}.
The contribution of the present work is to provide a margin bound (and corresponding Rademacher analysis)
that can be adapted to various operator norms at each layer.
Additionally, the present work operates in the multiclass setting,
and avoids an explicit dependence on the number of classes $k$,
which seems to appear in prior work
\citep{tz_multiclass_consistency,bartlett_tewari__multiclass}.

There are numerous generalization bounds for neural networks, including VC-dimension and fat-shattering bounds
(many of these can be found in \citep{anthony_bartlett_nn}).
Scale-sensitive analysis of neural networks started with \citep{bartlett_margin},
which can be interpreted in the present setting
as utilizing data norm $\|\cdot\|_\infty$ and operator norm $\|\cdot\|_{\infty\to\infty}$
(equivalently, the norm $\|A_{i}^{\top}\|_{1,\infty}$ on weight matrix $A_i$).
This analysis can be adapted to give a Rademacher complexity analysis \citep{bartlett_mendelson_rademacher},
and has been adapted to other norms \citep{neyshabur2015norm},
although the $\|\cdot\|_\infty$ setting appears to be necessary to avoid extra combinatorial factors.
More work is still needed to develop complexity analyses that have matching upper and lower bounds,
and also to determine which norms are well-adapted to neural networks as used in practice.

The present analysis utilizes covering numbers,
and is most closely connected to earlier covering number bounds \citep[Chapter 12]{anthony_bartlett_nn},
themselves based on the earlier fat-shattering analysis \citep{bartlett_margin},
however the technique here of pushing an empirical cover through layers
is akin to VC dimension proofs for neural networks
\citep{anthony_bartlett_nn}.
The use of Maurey's sparsification lemma
was inspired
by linear predictor covering number bounds \citep{tz_covering}.

\textbf{Comparison to preprint.}
The original preprint of this paper \citep{bartlett2017spectrally} featured a slightly different version of the spectral complexity $R_{\cA}$, given by $
  \del{\prod_{i=1}^L \rho_i \|A_i\|_\sigma}
  \del{\sum_{i=1}^L \frac{ \|A_i - M_i\|_{1}^{2/3}}{\|A_i\|_\sigma^{2/3}}}^{3/2}.
$
In the present version \pref{eq:spec_comp}, each $\|A_i - M_i\|_{1}$ term is replaced by $\|A_i^{\top} - M_i^{\top}\|_{2,1}$. This is a strict improvement since for any matrix $A\in\bbR^{d\times{}m}$ one has $\nrm*{A}_{2,1}\leq{}\nrm*{A}_{1}$, and in general the gap between these two norms can be as large as $\sqrt{d}$.

On a related note, all of the figures in this paper use the $\ls_1$ norm in the spectral complexity $R_\cA$ instead of the $(2,1)$ norm. Variants of the experiments described in \pref{sec:empirical} were carried out using each of the $l_1$, $(2,1)$, and $l_2$ norms in the $(\sum_{i=1}^{L}(\cdot)^{2/3})^{3/2}$ term with negligible difference in the results.

Since spectrally-normalized margin bounds were first proposed in the preprint \citep{bartlett2017spectrally}, subsequent works \citep{neyshabur2017pac,neyshabur2017implicit}  re-derived a similar spectrally-normalized bound using the PAC-Bayes framework. Specifically, these works showed that $R_{\cA}$ may be replaced (up to $\ln(W)$ factors) by:
  $\del[1]{\prod_{i=1}^L \rho_i \|A_i\|_\sigma}\cdot
  L\del[1]{\sum_{i=1}^L \frac{(\sqrt{W}\|A_i - M_i\|_{2})^{2}}{\|A_i\|_\sigma^{2}}}^{1/2}$.
Unfortunately, this bound never improves on \Cref{fact:main:new}, and indeed can be derived from it as follows.
First, the dependence on the individual matrices $A_i$ in the second term of this bound can be obtained from \Cref{fact:main:new} since for any $A\in\bbR^{d\times{}m}$ it holds that $\nrm*{A^{\top}}_{2,1}\leq{}\sqrt{d}\nrm*{A}_{2}$. Second, the functional form $(\sum_{i=1}^{L}(\cdot)^{2/3})^{3/2}$ appearing in \Cref{fact:main:new} may be replaced by the form $L(\sum_{i=1}^{L}(\cdot)^{2})^{1/2}$ appearing above by using $\nrm*{\alpha}_{2/3}\leq{}L\nrm*{\alpha}_{2}$ which holds for any $\alpha\in\bbR^{L}$ (and can be proved, for instance, with Jensen's inequality).

\section{Further observations and open problems}
\label{sec:open}

\paragraph{Adversarial examples.}
Adversarial examples are a phenomenon where the neural network predictions can be altered by
adding seemingly imperceptible noise to an input \citep{nn_adversarial}.
This phenomenon can be connected to margins as follows.  The margin is nothing more than
the distance an input must traverse before its label is flipped; consequently, low margin points
are more susceptible to adversarial noise than high margin points.
Concretely, taking the 100 lowest margin inputs from \ciften and adding uniform noise
at scale $0.15$ yielded flipped labels on 5.86\% of the images, whereas the same level
of noise on high margin points yielded 0.04\% flipped labels.
Can the bounds here suggest a way to defend against adversarial examples?

\paragraph{Regularization.}
It was observed in \citep{rethinking} that explicit regularization contributes little to the generalization
performance of neural networks.  In the margin framework, standard weight decay ($l_2$) regularization
seemed to have little impact on margin distributions in \Cref{sec:empirical}.
On the other hand, in the boosting literature, special types of regularization were developed to
maximize margins \citep{shai_singer_weaklearn_linsep}; perhaps a similar development can be performed here?

\paragraph{SGD.}
The present analysis applies to predictors that have large margins; what is missing is an analysis verifying
that SGD applied to standard neural networks returns large margin predictors!
Indeed, perhaps SGD returns not simply large margin predictors, but predictors that are well-behaved
in a variety of other ways that can be directly translated into refined generalization bounds.

\paragraph{Improvements to \Cref{fact:main:new}.}
There are several directions in which \Cref{fact:main:new} might be
improved. Can a better choice
of layer geometries (norms) yield better
bounds on practical networks?
Can the nonlinearities' worst-case Lipschitz constant
be replaced with an (empirically) averaged quantity?
Alternatively, can better lower bounds rule out these directions?

\paragraph{Rademacher vs. covering.}
Is it possible to prove \Cref{fact:main:new} solely via Rademacher complexity,
with no invocation of covering numbers?

\subsection*{Acknowledgements}
The authors thank Srinadh Bhojanapalli, Ryan Jian, Behnam Neyshabur, Maxim Raginsky,
Andrew J. Risteski, and Belinda Tzen for useful conversations
and feedback.
The authors thank Ben Recht for giving a provocative lecture at the Simons Institute,
stressing the need for understanding of both generalization and optimization of neural networks.
M.T.~and D.F.~acknowledge the use of a GPU machine provided by Karthik
Sridharan and made possible by an NVIDIA GPU grant.
D.F.~acknowledges the support of the NDSEG fellowship.
P.B.~gratefully acknowledges the support of the NSF through grant
IIS-1619362 and of the Australian Research Council through an
Australian Laureate Fellowship (FL110100281) and through the ARC
Centre of Excellence for Mathematical and Statistical Frontiers.
The authors thank the Simons Institute for the Theory of Computing
Spring 2017 program on the Foundations of Machine Learning.
Lastly, the authors are grateful to La Burrita (both the north and the south Berkeley campus locations)
for upholding the glorious tradition of the California Burrito.

\bibliographystyle{plainnat}
\bibliography{nn}

\clearpage

\appendix

\section{Proofs}

This appendix collects various proofs omitted from the main text.

\subsection{Lipschitz properties of ReLU and max-pooling nonlinearities}
\label{sec:app:prelim}

The standard \emph{ReLU} (``Rectified Linear Unit'') is the univariate mapping
\[
  \srelu(r) := \max\{0, r\}.
\]
When applied to a vector or a matrix, it operates coordinate-wise.
While the ReLU is currently the most popular choice of univariate nonlinearity,
another common choice is the \emph{sigmoid} $r\mapsto 1 / (1+\exp(-r))$.
More generally, these univariate nonlinearities are Lipschitz, and this
carries over to their vector and matrix forms as follows.

\begin{lemma}
  \label{fact:relu_lip}
  If $\sigma : \R^d\to\R^d$ is $\rho$-Lipschitz along every coordinate,
  then it is $\rho$-Lipschitz according to $\|\cdot\|_p$ for any $p\geq 1$.
  \end{lemma}
\begin{proof}
  for any $z,z'\in\R^d$,
  \[
    \|\sigma(z) - \sigma(z')\|_p
    = \del{ \sum_i |\sigma(z)_i - \sigma(z')_i|^p }^{1/p}
    \leq \del{ \sum_i \rho^p|z_i - z'_i|^p }^{1/p}
    = \rho\|z - z'\|_p.
  \]
\end{proof}

Define a \emph{max-pooling operator} $\cP$ as follows.
Given an input and output pair of finite-dimensional vector spaces $\cT$ and $\cT'$ (possibly arranged as matrices or tensors),
the max-pooling operator iterates over a collection of sets of indices $\cZ$
(whose cardinality is equal to the dimension of $\cT$'),
and for each element of $Z_i \in \cZ$ sets the corresponding coordinate $i$ in the output
to the maximum entry of the input over $Z_i$:
given $T \in \cT$,
\[
  \cP(T)_i := \max_{j \in Z_i} T_j.
\]
The following Lipschitz constant of pooling operators will depend on the number of times each coordinate
is accessed across elements of $\cZ$;
when this operator is used in computer vision, the number of times is typically a small constant,
for instance 5 or 9 \citep{imagenet_sutskever}.
\begin{lemma}
  \label{fact:maxpool_lip}
  Suppose that each coordinate $j$ of the input appears in at most $m$ elements of the collection $\mc{Z}$. Then
  the max-pooling operator $\cP$
  is $m^{1/p}$-Lipschitz 
  wrt $\|\cdot\|_p$ for any $p\geq 1$. In particular, the max-pooling operator is $1$-Lipschitz whenever $\mc{Z}$ forms a partition.
  \end{lemma}
\begin{proof}
  Let $T,T' \in \cT$ be given.
  First consider any fixed set of indices $Z \in \cZ$,
  and suppose without loss of generality that $\cP(T)_Z = \max_{j \in Z} T_j \geq \max_{j \in Z} T'_j$.
  Then
  \[
    |\cP(T)_Z - \cP(T')_Z|^p
    = \del{ \min_{j' \in Z} \max_{j \in Z} T_j - T'_{j'} }^p
    \leq \max_{j \in Z} \del{ T_j - T'_j }^p
    \leq \sum_{j \in Z} \envert{ T_j - T'_j }^p.
  \]
  Consequently,
  \begin{align*}
    \|\cP(T) - \cP(T')\|_p
    &
    =
    \del{ \sum_i |\cP(T)_i - \cP(T')_i|^p }^{1/p}
    =
    \del{ \sum_{Z \in \cZ} |\cP(T)_Z - \cP(T')_Z|^p }^{1/p}
    \\
    &\leq
    \del{ \sum_{Z \in \cZ} \sum_{j \in Z} |T_j - T'_j|^p }^{1/p}
    =
    \del{ \sum_{j}\sum_{Z\in\mc{Z}:j\in{}Z} |T_j - T'_j|^p }^{1/p}
    \\
    &\leq{}
    \del{ m\sum_{j}|T_j - T'_j|^p }^{1/p}
    =
    m^{1/p}\|T - T'\|_p.
  \end{align*}
\end{proof}

\subsection{Margin properties in \Cref{sec:multiclass_margins}}

The goal of this subsection is to prove the general margin bound in \Cref{fact:margin_multiclass}.
To this end, it is first necessary to establish a few properties
of the margin operator $\mop(v,j) := v_j - \max_{i\neq j} v_i$
and of the ramp loss $\ell_\lambda$.

\begin{lemma}
  \label{fact:margin_lip}
  For every $j$ and every $p\geq 1$, $\mop(\cdot, j)$ is 2-Lipschitz wrt $\|\cdot\|_p$.
\end{lemma}
\begin{proof}
  Let $v,v',j$ be given, and suppose (without loss of generality) $\mop(v,j) \geq \mop(v',j)$.
  Choose coordinate $i\neq j$ so that $\mop(v',j) = v'_j - v'_i$.
  Then
  \begin{align*}
    \mop(v,j) - \mop(v',j)
    &= \del{ v_j - \max_{l\neq j} v_j } - \del{ v'_j - v'_i }
    = v_j - v'_j + v'_i + \min_{l \neq j } (-v_l)
    \\
    &\leq \del{ v_j - v'_j} + \del{v'_i -v_i }
    \leq 2\|v-v'\|_\infty
    \leq 2\|v-v'\|_p.
  \end{align*}
\end{proof}

Next, recall the definition of the ramp loss
\[
  \ell_\gamma(r) :=
  \begin{cases}
    0
    &r< -\gamma,
    \\
    1 + r/\gamma
    &r \in [-\gamma,0],
    \\
    1
    &
    r > 0,
  \end{cases}
\]
and of the ramp risk
\[
  \cR_\gamma(f) := \bbE(\ell_\gamma(-\mop(f(x),y))).
\]
(These quantities are standard; see for instance \citep{esaim_survey,tz_multiclass_consistency,bartlett_tewari__multiclass}.)

\begin{lemma}
  \label{fact:margins_ramp}
  For any $f : \R^d \to \R^k$ and every $\gamma>0$,
  \[
    \Pr[ \argmax_i f(x)_i \neq y]
    \leq
    \Pr[ \mop(f(x),y) \leq 0 ]
    \leq
    \cR_\gamma(f),
  \]
  where the $\argmax$ follows any deterministic tie-breaking strategy.
\end{lemma}
\begin{proof}
  \begin{align*}
    \Pr[ \argmax_i f(x)_i \neq y]
    &\leq
    \Pr[ \max_{i\neq y} f(x)_i \geq f(x)_y]
    \\
    &=
    \Pr[ -\mop(f(x),y) \geq 0 ]
    \\
    &=
    \bbE \1[ -\mop(f(x),y) \geq 0 ]
    \\
    &\leq
    \bbE \ell_\gamma(-\mop(f(x),y))
  \end{align*}
\end{proof}

With these tools in place, the proof of \Cref{fact:margin_multiclass} is straightforward.
\begin{proof}[Proof of \Cref{fact:margin_multiclass}.]
  Since $\ell_\gamma$ has range $[0,1]$, it follows by standard properties of Rademacher complexity
  \citep[see, for example,][Theorem 3.1]{mohri_book}
  that with probability at least $1-\delta$,
  every $f\in\cF$ satisfies
  \[
    \cR_{\gamma}(f) \leq  \hcR_\gamma (f)
    +
    2 \Rad((\cF_\gamma)_{|S})
    + 3\sqrt{\frac{\ln(2/\delta)}{2n}}.
  \]
  The bound now follows by applying \Cref{fact:margins_ramp} to the left hand side.

\end{proof}

\subsection{Dudley Entropy Integral}

This section contains a slight variant of the standard Dudley entropy integral bound on the empirical Rademacher complexity (e.g. \cite{mohri_book}), which is used in the proof of \Cref{fact:main:new}. The presentation here diverges from standard presentations because the data metric (as in \cref{eq:data_l2}) is not normalized by $\sqrt{n}$. The proof itself is entirely standard however --- even up to constants --- and is included only for completeness.

\begin{lemma}
\label{lem:dudley}
Let $\mc{F}$ be a real-valued function class taking values in $\brk{0,1}$, and assume that $\mathbf{0}\in\mc{F}$. Then
\[
    \Rad(\cF_{|S})
    \leq \inf_{\alpha>0}\del{
      \frac{4\alpha}{\sqrt{n}} + \frac{12}{n}\int_{\alpha}^{\sqrt{n}}\sqrt{\log\cN(\mc{F}_{|S}, \veps, \nrm{\cdot}_{2})}d\veps.
    }
\]
\end{lemma}
\begin{proof}
Let $N\in\bbN$ be arbitrary and let $\veps_i=\sqrt{n}2^{-(i-1)}$ for each $i\in\brk{N}$. For each $i$ let $V_{i}$ denote the cover achieving $\cN(\mc{F}_{|S}, \veps_i, \nrm{\cdot}_{2})$, so that
\begin{equation}
\label{eq:data_l2}
\forall{}f\in\mc{F}\quad\exists{}v\in{}V_{i}\quad\prn*{\sum_{t=1}^{n}(f(x_t) - v_t)^{2}}^{1/2} \leq{} \veps_i,
\end{equation}
and $\abs{V_{i}} = \cN(\mc{F}_{|S}, \veps_i, \nrm{\cdot}_{2})$. For a fixed $f\in\cF$, let $v^{i}[f]$ denote the nearest element in $V_{_i}$. Then
\begin{align*}
&\En_{\eps}\sup_{f\in\cF}\sum_{t=1}^{n}\veps_if(x_t) \\&= \En_{\eps}\sup_{f\in\cF}\brk*{
\sum_{t=1}^{n}\eps_{t}(f(x_t)-v_{t}^{N}[f]) + \sum_{i=1}^{N-1}\sum_{t=1}^{n}\eps_{t}(v_{t}^{i}[f] - v_{t}^{i+1}[f]) - \sum_{t=1}^{n}\eps_{t}v_{t}^{1}[f]
} \\
&\leq{} 
\En_{\eps}\sup_{f\in\cF}\brk*{\sum_{t=1}^{n}\eps_{t}(f(x_t)-v_{t}^{N}[f])} + \sum_{i=1}^{N-1}\En_{\eps}\sup_{f\in\cF}\brk*{\sum_{t=1}^{n}\eps_{t}(v_{t}^{i}[f] - v_{t}^{i+1}[f])} + \En_{\eps}\sup_{f\in\cF}\brk*{\sum_{t=1}^{n}\eps_{t}v_{t}^{1}[f]}.
\end{align*}
For the third term, observe that it suffices to take $V_{1}=\crl*{\boldsymbol{0}}$, which implies
\[
\En_{\eps}\sup_{f\in\cF}\brk*{\sum_{t=1}^{n}\eps_{t}v_{t}^{1}[f]} = 0.
\]
The first term may be handled using Cauchy-Schwarz as follows:
\[
\En_{\eps}\sup_{f\in\cF}\brk*{\sum_{t=1}^{n}\eps_{t}(f(x_t)-v_{t}^{N}[f])}
\leq{} \sqrt{\En_{\eps}\sum_{t=1}^{n}(\eps_t)^{2}}\sqrt{\sup_{f\in\cF}\sum_{t=1}^{n}(f(x_t)-v_{t}^{N}[f])^{2}
} \leq{} \sqrt{n}\veps_{N}.
\]
Last to take care of are the terms of the form
\[
\En_{\eps}\sup_{f\in\cF}\brk*{\sum_{t=1}^{n}\eps_{t}(v_{t}^{i}[f] - v_{t}^{i+1}[f])}.
\]
For each $i$, let $W_i=\crl*{v^{i}[f]-v^{i+1}[f]\mid{}f\in\cF}$. Then $\abs{W_i}\leq{}\abs{V_i}\abs{V_{i+1}}\leq{}\abs{V_{i+1}}^{2}$,
\[
\En_{\eps}\sup_{f\in\cF}\brk*{\sum_{t=1}^{n}\eps_{t}(v_{t}^{i}[f] - v_{t}^{i+1}[f])} \leq{} \En_{\eps}\sup_{w\in{}W_i}\brk*{\sum_{t=1}^{n}\eps_tw_{t}},
\]
and furthermore
\begin{align*}
\sup_{w\in{}W_i}\sqrt{\sum_{t=1}^{n}w_{t}^{2}}  & = \sup_{f\in\cF}\nrm*{v^{i}[f] - v^{i+1}[f]}_{2} \\
&\leq{} 
\sup_{f\in\cF}\nrm*{v^{i}[f] - (f(x_1),\ldots,f(x_{n}))}_{2} + \sup_{f\in\cF}\nrm*{(f(x_1),\ldots,f(x_{n})) - v^{i+1}[f]}_{2}\\
&\leq{}\veps_{i} + \veps_{i+1}\\
&= 3\veps_{i+1}.
\end{align*}
With this observation, the standard Massart finite class lemma \citep{mohri_book} implies
\[
\En_{\eps}\sup_{w\in{}W_i}\brk*{\sum_{t=1}^{n}\eps_tw_{t}} \leq{}\sqrt{2\sup_{w\in{}W_i}\sum_{t=1}^{n}(w_t)^{2}\log\abs{W_i}} \leq{} 3\sqrt{2\log\abs{W_i}}\veps_{i+1} \leq{} 6\sqrt{\log\abs{V_{i+1}}}\veps_{i+1}.
\]
Collecting all terms, this establishes 
\begin{align*}
\En_{\eps}\sup_{f\in\cF}\sum_{t=1}^{n}\eps_{t}f(x_t) &\leq{} \veps_{N}\sqrt{n} + 6\sum_{i=1}^{N-1}\veps_{i+1}\sqrt{\log\cN(\mc{F}_{|S}, \veps_{i+1}, \nrm{\cdot}_{2})} \\
&\leq{} \veps_{N}\sqrt{n} + 12\sum_{i=1}^{N}(\veps_{i}-\veps_{i+1})\sqrt{\log\cN(\mc{F}_{|S}, \veps_i, \nrm{\cdot}_{2})} \\
&\leq{} \veps_{N}\sqrt{n} + 12\int_{\veps_{N+1}}^{\sqrt{n}}\sqrt{\log\cN(\mc{F}_{|S}, \veps, \nrm{\cdot}_{2})}d\veps.
\end{align*}
Finally, select any $\alpha>0$ and take $N$ be the largest integer with $\veps_{N+1}>\alpha$. Then $\veps_{N}=4\veps_{N+2}<4\alpha$, and so
\[
\veps_{N}\sqrt{n} + 12\int_{\veps_{N+1}}^{\sqrt{n}}\sqrt{\log\cN(\mc{F}_{|S}, \veps, \nrm{\cdot}_{2})}d\veps \leq{} 4\alpha\sqrt{n} + 12\int_{\alpha}^{\sqrt{n}}\sqrt{\log\cN(\mc{F}_{|S}, \veps, \nrm{\cdot}_{2})}d\veps.
\]

\end{proof}

\subsection{Proof of matrix covering (\Cref{fact:matrix_l21_covering})}

First recall the Maurey sparsification lemma.

\begin{lemma}[name={Maurey; cf. \citep{pisier1980remarques}, \citep[Lemma 1]{tz_covering}}]
  \label{fact:maurey}
  Fix Hilbert space $\cH$ with norm $\|\cdot\|$.
  Let $U\in \cH$ be given with representation $U = \sum_{i=1}^d \alpha_i V_i$ where $V_i\in\cH$ and $\alpha \in \R^d_{\geq 0} \setminus \cbr{0}$.
  Then for any positive integer $k$,
  there exists a choice of nonnegative integers $(k_1, \ldots, k_d)$,
  $\sum_i k_i = k$, such that
  \[
    \enVert{ U - \frac {\|\alpha\|_1}{k} \sum_{i=1}^d k_i V_i }^2
      \leq 
     \frac {\|\alpha\|_1} k \sum_{i=1}^d \alpha_i \| V_i \|^2
    \leq
    \frac {\|\alpha\|_1^2} k \max_i \|V_i\|^2.
  \]
\end{lemma}
\begin{proof}
  Set $\beta := \|\alpha\|_1$ for convenience,
  and let $(W_1, \ldots, W_k)$ denote $k$ iid random variables
  where $\Pr[W_1 = \beta V_i] := \alpha_i/\beta$.
  Define $W := k^{-1} \sum_{i=1}^k W_i$,
  whereby
  \[
    \bbE W = \bbE W_1 = \sum_{i=1}^d \beta V_i \del{ \frac{\alpha_i}{\beta} }
    = U.
  \]
  Consequently
  \begin{align*}
    \bbE \|U - W\|^2
    &=
    \frac 1 {k^2}
    \bbE \enVert{\sum_i (U - W_i)}^2
    =
    \frac 1 {k^2}
    \bbE
    \del{
      \sum_i
      \|U - W_i\|^2
      +
      \sum_{i \neq j}
      \ip{U - W_i}{U - W_j}
    }
    \\
    &= \frac 1 k \bbE\|U - W_1\|^2
    = \frac 1 k \del{ \bbE\| W_1\|^2 - \|U\|^2 }
    \leq \frac 1 k \bbE\| W_1\|^2
    \\
    &= \frac 1 k \sum_{i=1}^d \frac {\alpha_i}{\beta} \| \beta V_i\|^2
    = \frac \beta k \sum_{i=1}^d \alpha_i \| V_i\|^2
    \\
    &\leq
    \frac {\beta^2} k \max_i \| V_i\|^2.
  \end{align*}
  To finish, by the probabilistic method,
  there exists integers $(j_1,\ldots, j_k) \in \{1,\ldots,d\}^{k}$
  and an assignment $\hW_i := \beta V_{j_i}$ and $\hW := k^{-1} \sum_{i=1}^{k}\hW_i$
  such that
  \[
    \enVert{U - \hW}^2
    \leq \bbE\enVert{U - W}^2.
  \]
  The result now follows by defining integers $(k_1,\ldots,k_d)$ according
  to $k_i := \sum_{l=1}^k \1[j_l = i]$.
\end{proof}

As stated, the Maurey sparsification lemma seems to only grant bounds
in terms of $l_1$ norms.
As developed by \citet{tz_covering} in the vector covering case,
however, it is easy to
handle other norms by rescaling the cover elements.
With slightly more care, these proofs generalize to the matrix case,
thus yielding the proof of \Cref{fact:matrix_l21_covering}.

\begin{proof}[Proof of \Cref{fact:matrix_l21_covering}]
  Let matrix $X\in \R^{n\times d}$ be given,
  and obtain matrix $Y \in \R^{n\times d}$ by rescaling the columns of $X$
  to have unit $p$-norm: $Y_{:, j} := X_{:,j} / \|X_{:,j}\|_p$.
  Set $N:= 2dm$
  and $k := \lceil a^2 b^2 m^{2/r} / \epsilon^2 \rceil$
  and $\bar a := a m^{1/r}\|X\|_p$,
  and define
  \begin{align}
    \cbr{V_1,\ldots, V_N}
    &:= \cbr{ gY\bfe_i\bfe_j^\top : g\in \cbr{-1,+1}, i \in \cbr{1,\ldots,d}, j \in \cbr{1,\ldots,m} },
    \notag
    \\
    \cC
    &:=
    \cbr{
      \frac {\bar a}{k}
      \sum_{i=1}^N k_i V_i : k_i \geq 0, \sum_{i=1}^N k_i = k
    }
    =
    \cbr{
      \frac {\bar a} k
      \sum_{j=1}^kV_{i_j}
      :
      (i_1,\ldots,i_k) \in \brk{N}^k
    },
    \label{eq:whatever_man}
  \end{align}
  where the $k_i$'s are integers. Now $p\leq 2$ combined with the definition of $V_i$ and $Y$ implies
  \[
    \max_i \|V_i\|_2 \leq \max_i \|Y\bfe_i\|_2 = \max_i \frac {\|X\bfe_i\|_2}{\|X\bfe_i\|_p} \leq 1.
  \]

  It will now be shown that $\cC$ is the desired cover.
  Firstly, $|\cC| \leq N^k$ by construction, namely by the final equality of \cref{eq:whatever_man}.
    Secondly, let $A$ with $\|A\|_{q,s}\leq a$ be given,
  and construct a cover element within $\cC$ using the following technique,   which follows the approach developed by \citet{tz_covering} for linear prediction in which the basic Maurey lemma is applied to non-$l_1$ balls simply by rescaling.

  \begin{itemize}
    \item
      Define $\alpha \in \R^{d \times m}$ to be a ``rescaling matrix'' where every
      element of row $j$ is equal to $\|x_j\|_p$;
      the purpose of $\alpha$ is to annul the rescaling of $X$ introduced by $Y$,
      meaning $XA = Y(\alpha \odot A)$ where ``$\odot$'' denotes element-wise product.
      Note,
      \begin{align*}
        \|\alpha\|_{p,r}
        &= \enVert{ (\|\alpha_{:,1}\|_p, \ldots, \|\alpha_{:,m} \|_p)}_r \\
       & = \enVert{ \prn*{ \enVert{ (\|X_{:,1}\|_p,\ldots,\|X_{:,d}\|_p) }_p, \ldots, \enVert{ (\|X_{:,1}\|_p,\ldots,\|X_{:,d}\|_p) }_p }}_r
        \\
        &= m^{1/r} \enVert{ (\|X_{:,1}\|_p,\ldots,\|X_{:,d}\|_p) }_p
        = m^{1/r} \del{\sum_{j=1}^d \|X_{:,j}\|_p^p}^{1/p}
        \\
        &= m^{1/r} \del{\sum_{j=1}^d \sum_{i=1}^n X_{i,j}^p }^{1/p}
        = m^{1/r} \|X\|_p.
      \end{align*}
    \item
      Define $B := \alpha \odot A$, whereby using conjugacy of $\|\cdot\|_{p,r}$ and $\|\cdot\|_{q,s}$ gives
      \begin{align*}
        \|B\|_1 \leq \ip{\alpha}{|A|} \leq \|\alpha\|_{p,r} \|A\|_{q,s} \leq m^{1/r} \|X\|_p a
        = \bar a.
      \end{align*}
      Consequently, $XA$ is equal to 
      \[
	YB
        = Y \sum_{i=1}^d \sum_{j=1}^m B_{ij} \bfe_i \bfe_j^\top
        = \|B\|_1 \sum_{i=1}^d \sum_{j=1}^m \frac{B_{ij}}{\|B\|_1} \del{ Y \bfe_i \bfe_j^\top}
        \in
        \bar a\cdot \conv(\cbr{V_1,\ldots,V_N}),
      \]
      where $\conv(\cbr{V_1,\ldots,V_N})$ is the convex hull of $\cbr{V_1,\ldots,V_N}$.
    \item
      Combining the preceding constructions with \Cref{fact:maurey},
      there exist nonnegative integers $(k_1,\ldots,k_N)$ with $\sum_i k_i = k$
      with
      \[
        \enVert{XA - \frac {\bar a}{k} \sum_{i=1}^N k_i V_i}_{2}^2
        =
        \enVert{YB - \frac {\bar a}{k} \sum_{i=1}^N k_i V_i}_{2}^2
        \leq \frac {{\bar a}^2}{k}    \max_i \|V_i\|_2
        \leq \frac {a^2 m^{2/r} \|X\|_p^2}{k}
        \leq \eps^2.
      \]
      The desired cover element is thus $\frac {\bar a}{k} \sum_i k_i V_i \in \cC$.
      \end{itemize}
\end{proof}

\subsection{A whole-network covering bound for general norms}
\label{sec:cover:general}

As stated in the text, the construction of a whole-network cover via
induction on layers does not demand much structure from the norms placed on the weight matrices.
This subsection develops this general analysis.
A tantalizing direction for future work is to specialize the general bound in other ways,
namely ones that are better adapted to the geometry of neural networks as encountered
in practice.

The structure of the networks is the same as before;
namely, given matrices $\cA = (A_1,\ldots,A_L)$,
define the mapping $F_\cA$ as~\eqref{eq:FA}, and more generally for
$i\le L$ define $\cA_1^i:=(A_1,\ldots,A_i)$ and
\[
  F_{\cA_1^i}(Z) := \sigma_i(A_i \sigma_{i-1}(A_{i-1} \cdots \sigma_1(A_1 Z) \cdots)),
\]
with the convention $F_{\emptyset}(Z)=Z$.
\begin{itemize}
  \item
    Define two sequences of vector spaces $\cV_1,\ldots,\cV_L$ and
    $\cW_2,\ldots,\cW_{L+1}$, where $\cV_i$ has a norm $|\cdot|_i$ and
    $\cW_i$ has norm $\opnorm{\cdot}_i$.
  \item
    The inputs $Z \in \cV_1$ satisfy a norm constraint $|Z|_1 \leq B$.
    The subscript merely indicates an index, and does not refer to any $l_1$ norm.
    The vector space $\cV_1$, and moreover the collection of vector spaces
    $\cV_i$ and $\cW_i$, have no fixed meaning and are simply abstract vector spaces.
    However, when using these tools to prove \Cref{fact:main:new},
    $\cV_1 = \R^{d\times n}$ and $Z\in \cV_1$ is formed by collecting the $n$ data points
    into its columns; that is, $Z=X^{\top}$.
  \item
    The linear operators $A_i : \cV_i \to \cW_{i+1}$ are associated with
    some operator norm $|A_i|_{i\to i+1} \leq c_i$:
    \[
      |A_i|_{i\to i+1} := \sup_{|Z|_i \leq 1} \opnorm{A_i Z}_{i+1} = c_i.
    \]
    As stated before, these linear operators $\cA = (A_1,\ldots,A_L)$ vary across functions $F_\cA$.
    When used to prove \Cref{fact:main:new},
    $Z$ is a matrix (the forward image of data matrix $X^\top$ across layers),
    and these norms are all matrix norms.
  \item
    The $\rho_i$-Lipschitz mappings $\sigma_i : \cW_{i+1} \to \cV_{i+1}$ have $\rho_i$ measured with
    respect to norms $|\cdot|_{i+1}$ and $\opnorm{\cdot}_{i+1}$:
    for any $z,z' \in \cW_{i+1}$,
    \[
      \envert{\sigma_i(z) - \sigma_i(z')}_{i+1} \leq \rho_i \opnorm{ z - z' }_{i+1}.
    \]
    These Lipschitz mappings are considered fixed within $F_\cA$.
    Note again that these operations, when applied to prove
    \Cref{fact:main:new}, operate on matrices that represent the forward images of all data points
    together.  Lipschitz properties of the standard coordinate-wise ReLU and max-pooling operators
    can be found in \Cref{sec:app:prelim}.
\end{itemize}

\begin{lemma}
  \label{fact:cover:general}
  Let $(\eps_1,\ldots,\eps_L)$ be given,
  along with fixed Lipschitz mappings $(\sigma_1,\ldots,\sigma_L)$ (where $\sigma_i$ is $\rho_i$-Lipschitz),
  and operator norm bounds $(c_1,\ldots,c_L)$.
  Suppose the matrices $\cA = (A_1,\ldots,A_L)$ lie within
  $\cB_1\times\cdots\times \cB_L$ where $\cB_i$ are arbitrary classes with the property that each $A_i \in \cB_i$ has $|A_i|_{i\to i+1}\leq c_i$.
  Lastly, let data $Z$ be given with $|Z|_1\leq B$.
  Then, letting  $\tau := \sum_{j\leq L} \eps_j \rho_j \prod_{l=j+1}^L \rho_l c_l$, the neural net images $\cH_Z := \{ F_\cA(Z) : \cA \in \cB_1\times\cdots\times \cB_L\}$
  have covering number bound
    \begin{align*}
    \cN\del{\cH_Z, \tau, |\cdot|_{L+1}}
    \leq
    &
    \prod_{i=1}^{L}
        \sup_{\substack{(A_1,\ldots,A_{i-1}) \\ \forall j <i \centerdot A_j \in \cB_j}}
    \cN\del{ \cbr{ A_i F_{(A_1,\ldots,A_{i-1})}(Z)
    :
    A_i \in \cB_i}, \eps_i, \opnorm{\cdot}_{i+1} }.
    \end{align*}
  \end{lemma}
\begin{proof}  Inductively construct covers $\cF_1,\ldots,\cF_L$ of $\cW_{2},\ldots,\cW_{L+1}$
    as follows.
  \begin{itemize}
    \item
      Choose an $\eps_1$-cover $\cF_1$
      of $\cbr{ A_1 Z : A_1 \in \cB_1}$,
      thus
      \[
        |\cF_1| \leq \cN(
        \cbr{ A_1Z : A_1 \in \cB_1},
        \eps_1, \opnorm{\cdot}_{2})
        =: N_1.
      \]

    \item
      For every element $F\in\cF_i$,
      construct an $\eps_{i+1}$-cover $\cG_{i+1}(F)$ of
      \[
        \cbr{ A_{i+1} \sigma_{i}(F) : A_{i+1} \in \cB_{i+1} }.
      \]
      Since the covers are proper, meaning
      $F = A_iF_{(A_1,\ldots,A_{i-1})}(Z)$ for some matrices $(A_1,\ldots,A_i) \in \cB_1\times\cdots\times \cB_i$,
      it follows that
      \[
        \envert{\cG_{i+1}(F)}
        \leq
                \sup_{\substack{(A_1,\ldots,A_{i}) \\ \forall j\leq i \centerdot A_j \in \cB_j}}
        \cN\del{
          \cbr{ A_{i+1} F_{A_1,\ldots,A_i}(Z) : A_{i+1} \in \cB_{i+1} },
          \eps_{i+1},
          \opnorm{\cdot}_{i+2}
        }
        =: N_{i+1}.
      \]
      Lastly form the cover
      \[
        \cF_{i+1} := \bigcup_{F\in\cF_i} \cG_{i+1}(F),
      \]
      whose cardinality satisfies
      \[
        \envert{ \cF_{i+1} }
        \leq |\cF_i| \cdot N_{i+1}
        \leq \prod_{l=1}^{i+1} N_l.
      \]
  \end{itemize}

  Define $\cF := \cbr{ \sigma_L(F) : F \in \cF_{L} }$; by construction, $\cF$ satisfies the desired cardinality constraint.
  to show that it is indeed a cover,
  fix any $(A_1,\ldots,A_L)$ satisfying the above constraints,
  and for convenience define recursively the mapped elements
  \[
    F_1 = A_1 X \in \cW_2,
    \qquad
    G_i = \sigma_i(F_i) \in \cV_{i+1}
    \qquad
    F_{i+1} = A_{i+1} G_i \in \cW_{i+2} .
  \]
  The goal is to exhibit $\hG_L\in \cF$ satisfying
  $|G_L - \hG_L|_{L+1} \leq \tau$.
  To this end, inductively construct approximating elements $(\hF_i,\hG_i)$ as follows.
  \begin{itemize}
    \item
      Base case: set $\hG_0 = X$.
    \item
      Choose $\hF_i \in \cF_i$ with $\opnorm{A_i \hG_{i-1} - \hF_i}_{i+1} \leq \eps_i$,
      and set $\hG_i := \sigma_i(\hF_i)$. 
        \end{itemize}
  To complete the proof, it will be shown inductively that
  \begin{align*}
    |G_{i} - \hG_{i}|_{i+1}
    &\leq
    \sum_{1\leq{}j\leq i} \eps_j \rho_j \prod_{l=j+1}^i \rho_l c_l.
  \end{align*}
  For the base case,
  \[
    |G_0 - \hG_0|_1 = 0.
           \]
  For the inductive step,
  \begin{align*}
    |G_{i+1} - \hG_{i+1}|_{i+2}
    &\leq \rho_{i+1} \opnorm{F_{i+1} - \hF_{i+1}}_{i+2}
    \\
    &\leq \rho_{i+1} \opnorm{F_{i+1} - A_{i+1} \hG_i}_{i+2}
    + \rho_{i+1} \opnorm{A_{i+1} \hG_i - \hF_{i+1}}_{i+2}
    \\
    &\leq \rho_{i+1} \envert{A_{i+1}}_{i+1\to i+2} \envert{G_i - \hG_i}_{i+1}
    + \rho_{i+1} \eps_{i+1}
    \\
    &\leq \rho_{i+1} c_{i+1} \del{ \sum_{j\leq i} \eps_j \rho_j \prod_{l=j+1}^i \rho_l c_l }
    + \rho_{i+1} \eps_{i+1}
    \\
    &= \sum_{j\leq i+1} \eps_j \rho_j \prod_{l=j+1}^{i+1} \rho_l c_l.
  \end{align*}

\end{proof}

The core of the proof rests upon inequalities which break the task of covering a layer
into a cover term for the previous layer (handled by induction)
and another cover term for the present layer's weights (handled by matrix covering).
These inequalities are similar to those in an existing covering number
proof \citep[Chapter 12]{anthony_bartlett_nn} (itself rooted in the earlier work of \citet{bartlett_margin});
however that proof (a) operates node by node, and can not take advantage
of special norms on $\cA$, and
(b) does not maintain an empirical cover across layers, instead
explicitly covering the parameters of all weight matrices,
which incurs the number of parameters as a multiplicative factor.
The idea here to push an empirical cover through layers,
meanwhile, is reminiscent of VC dimension proofs for neural networks
\citep[Chapter 8]{anthony_bartlett_nn}.

\subsection{Proof of spectral covering bound (\Cref{fact:cover:spectral})}

The whole-network covering bound in terms of spectral and $(2,1)$ norms
now follows by the general norm covering number in
\Cref{fact:cover:general},
and the matrix covering lemma in \Cref{fact:matrix_l21_covering}.

\begin{proof}[Proof of \Cref{fact:cover:spectral}]
  First dispense with the parenthetical statement regarding
  coordinate-wise ReLU and max-pooling operaters,
  which are Lipschitz
  by \Cref{fact:relu_lip,fact:maxpool_lip}.
  The rest of the proof is now a consequence of \Cref{fact:cover:general}
  with all data norms set to the $l_2$ norm ($|\cdot|_i = \opnorm{\cdot}_i = \|\cdot\|_2$),
  all operator norms set to the spectral norm ($|\cdot|_{i\to i+1} = \|\cdot\|_\sigma$),
  the matrix constraint sets set to
  $\cB_{i}=\cbr{A_i : \|A_i\|_\sigma \leq s_i, \|A_i^{\top} - M_i^{\top}\|_{2,1} \leq b_i}$,
  and lastly the per-layer cover resolutions $(\eps_1,\ldots,\eps_L)$ set according to
  \begin{align*}
    \eps_i := \frac {\alpha_i\eps}{\rho_i \prod_{j > i} \rho_j s_j}
    \qquad
    \textup{where}
    \quad
    &\alpha_i := \frac 1 {\bar \alpha}\del {\frac {b_i}{s_i}}^{2/3},
    \quad
    \bar\alpha:= \sum_{j=1}^L \del {\frac {b_j}{s_j}}^{2/3}.
  \end{align*}
  By this choice, it follows that the final cover resolution $\tau$
  provided by \Cref{fact:cover:general} satisfies
  \begin{align*}
    \tau
    &\leq
    \sum_{j\leq L} \eps_j \rho_j \prod_{l=j+1}^L \rho_l s_l
    =
    \sum_{j\leq L} \alpha_j \eps
    = \eps.
  \end{align*}
  The key technique in the remainder of the proof is to apply \Cref{fact:cover:general} with the covering number estimate from
    \Cref{fact:matrix_l21_covering},
  but centering the covers at $M_i$ (meaning the cover at layer $i$ is of matrices
  $\cB_i$ where $A_i \in \cB_i$ satisfies $\|A_i^{\top}-M_i^{\top}\|_{2,1}\leq b_i$),
  and collecting $(x_1,\ldots,x_n)$ as rows of matrix $X\in\R^{n\times d}$.
  To start, the covering number estimate from \Cref{fact:cover:general}
  can be combined with
  \Cref{fact:matrix_l21_covering} (specifically with  $p=2$, $s=1$)
  to give
  \begin{align}
    & \ln \cN(\cH_{|S}, \eps, \|\cdot\|_2) \notag\\
    &\leq
    \sum_{i=1}^{L}
        \sup_{\substack{(A_1,\ldots,A_{i-1}) \\ \forall j <i \centerdot A_j \in \cB_j}}
    \ln
    \cN\del{ \cbr{ A_i F_{(A_1,\ldots,A_{i-1})}(X^\top)
    :
    A_i \in \cB_i}, \eps_i, \|\cdot\|_2 }
    \notag
    \\
    &\stackrel{(*)}{=}
    \sum_{i=1}^{L}
        \sup_{\substack{(A_1,\ldots,A_{i-1}) \\ \forall j <i \centerdot A_j \in \cB_j}}
    \ln
    \cN\del{ \cbr{ F_{(A_1,\ldots,A_{i-1})}(X^\top)^\top (A_i - M_i)^\top
    :
\|A_i^{\top} - M_i^{\top}\|_{2,1} \leq{} b_i,\ \nrm{A_i}_{\sigma}\leq{}s_i}, \eps_i, \|\cdot\|_2 }
    \notag
    \\
    &\leq{}
        \sum_{i=1}^{L}
        \sup_{\substack{(A_1,\ldots,A_{i-1}) \\ \forall j <i \centerdot A_j \in \cB_j}}
    \ln
    \cN\del{ \cbr{ F_{(A_1,\ldots,A_{i-1})}(X^\top)^\top (A_i - M_i)^\top
    :
    \|A_i^{\top} - M_i^{\top}\|_{2,1} \leq{} b_i}, \eps_i, \|\cdot\|_2 }
    \notag
    \\
    &\leq{}
    \sum_{i=1}^{L}
        \sup_{\substack{(A_1,\ldots,A_{i-1}) \\ \forall j <i \centerdot A_j \in \cB_j}}
    \frac {b_i^2 \|F_{(A_1,\ldots,A_{i-1})}(X^\top)^\top\|_{2}^2}{\eps_i^2} \ln(2W^2),
    \label{eq:cover_doit}
  \end{align}
  where $(*)$ follows firstly since $l_2$ covering a matrix and its transpose is the same,
  and secondly since
  the cover can be translated by $F_{(A_1,\ldots,A_{i-1})}(X^\top)^\top M_i^\top$
  without changing its cardinality.
  In order to simplify this expression, note for any $(A_1,\ldots,A_{i-1})$ that
  \begin{align*}
    \|F_{(A_1,\ldots,A_{i-1})}(X^\top)^\top\|_{2}
    &=
    \|F_{(A_1,\ldots,A_{i-1})}(X^\top)\|_{2}
    \\
    &=
    \|\sigma_{i-1}(A_{i-1}F_{(A_1,\ldots,A_{i-2})}(X^\top) - \sigma_{i-1}(0)\|_{2}
    \\
    &\leq
    \rho_{i-1}\|A_{i-1}F_{(A_1,\ldots,A_{i-2})}(X^\top) - 0\|_{2}
    \\
    &\leq
    \rho_{i-1}\|A_{i-1}\|_\sigma \|F_{(A_1,\ldots,A_{i-2})}(X^\top)\|_{2},
  \end{align*}
  which by induction gives
  \begin{align}
    \max_j\|F_{(A_1,\ldots,A_{i-1})}(X^\top)^\top\bfe_j\|_{2}
    &\leq
    \|X\|_2 \prod_{j=1}^{i-1} \rho_j  \|A_j\|_{\sigma}.
    \label{eq:matrix_unroll}
  \end{align}
  Combining \cref{eq:cover_doit,eq:matrix_unroll}, then expanding the choice of $\eps_i$ and collecting terms,
  \begin{align*}
    \ln \cN(\cH_{|S}, \eps, \|\cdot\|_2)
    &\leq
    \sum_{i=1}^{L}
        \sup_{\substack{(A_1,\ldots,A_{i-1}) \\ \forall j <i \centerdot A_j \in \cB_j}}
    \frac {b_i^2 \|X\|_2^2 \prod_{j<i} \rho_j^2 \|A_j\|_\sigma^2}{\eps_i^2} \ln(2W^2)
    \\
    &\leq
    \sum_{i=1}^{L}
    \frac {b_i^2 B^2 \prod_{j<i} \rho_j^2 s_j^2}{\eps_i^2} \ln(2W^2)
    \\
    &=
    \frac {B^2 \ln(2W^2)\prod_{j=1}^{L} \rho_j^2 s_j^2}{\eps^2}
    \sum_{i=1}^{L}
    \frac {b_i^2}{\alpha_i^2 s_i^2}
    \\
    &=
    \frac {B^2 \ln(2W^2)\prod_{j=1}^{L} \rho_j^2 s_j^2}{\eps^2}
    \del{\bar \alpha^3}.
  \end{align*}
\end{proof}

\subsection{Proof of \Cref{fact:main:new}}

As an intermediate step to \Cref{fact:main:new},
a bound is first produced which has
constraints on matrix and data norms provided in advance.

\begin{lemma}
  \label{fact:main:old}
  Let fixed nonlinearities $(\sigma_1,\ldots,\sigma_L)$
  and
  reference matrices $(M_1,\ldots, M_L)$
  be given where $\sigma_i$ is $\rho_i$-Lipschitz and $\sigma_i(0) = 0$.
  Further let margin $\gamma >0$, data bound $B$, spectral norm bounds $(s_i)_{i=1}^L$, and $l_1$ norm bounds $(b_i)_{i=1}^L$
  be given.
  Then with probability at least $1-\delta$ over an iid draw of $n$ examples $((x_i,y_i))_{i=1}^n$
  with $\sqrt{ \sum_i \|x_i\|_2^2} \leq B$,
  every network $F_\cA : \R^d \to \R^k$ whose weight matrices $\cA = (A_1,\ldots,A_L)$
  obey $\|A_i\|_\sigma \leq s_i$ and $\|A_i^{\top} - M_i^{\top}\|_{2,1} \leq b_i$
  satisfies
    \begin{align*}
    \Pr\sbr[2]{ \argmax_j F_\cA(x)_j \neq y }
    &\leq
    \hcR_\gamma(f)
    +
      \frac 8 n
      +
      \frac {72B \ln(2W) \ln(n)} {\gamma n}
      \del{\prod_{i=1}^L s_i \rho_i} \del{\sum_{i=1}^L \frac{b_i^{2/3}}{s_i^{2/3}}}^{3/2}
    +
    3 \sqrt{\frac{\ln(1/\delta)}{2n}}
    .
  \end{align*}
  \end{lemma}
\begin{proof}  Consider the class of networks $\cF_\lambda$ obtained by affixing the ramp loss
  $\ell_\gamma$ and the negated margin operator $-\mop$ to the output of the provided network class:
  \[
    \cF_\gamma := \cbr{ (x,y)\mapsto \ell_\gamma(-\mop(f(x),y)) : f \in \cF };
  \]
  Since $(z,y) \mapsto \ell_\gamma(-\mop(z,y))$ is $2/\gamma$-Lipschitz wrt $\|\cdot\|_2$ by
  \Cref{fact:margin_lip} and definition of $\ell_\gamma$,
  the function class $\cF_\gamma$ still falls under the setting of \Cref{fact:cover:spectral},
  and gives
  \[
    \ln \cN\del{(\cF_\gamma)_{|S}, \eps, \|\cdot\|_2}
    \leq
            \frac {4 B^2 \ln(2W^2)}{\gamma^2\eps^2}
    \del{
      \prod_{j=1}^L s_j^2\rho_j^2
    }
    \del{
      \sum_{i=1}^L \del{\frac {b_i}{s_i}}^{2/3}
    }^3
    =: \frac {R}{\eps^2}.
  \]
              What remains is to relate covering numbers and Rademacher complexity via a Dudley entropy integral;
  note that most presentations of this technique place $1/n$ inside the covering number norm,
  and thus the application here is the result of a tiny amount of massaging.
  Continuing with this in mind, the Dudley entropy integral bound on Rademacher complexity from \Cref{lem:dudley} grants
  \begin{align*}
    \Rad((\cF_\gamma)_{|S})
    &\leq \inf_{\alpha>0}\del{
      \frac{4\alpha}{\sqrt{n}} + \frac {12}{n} \int_{\alpha}^{\sqrt{n}} \sqrt{
        \frac {R}{\eps^2}
      }
      \dif\eps
    }
  =
              \inf_{\alpha>0}\del[4]{
      \frac{4\alpha}{\sqrt{n}} + \ln(\sqrt{n}/\alpha)       \frac {12\sqrt{R}}{n}
           }.
  \end{align*}
   The $\inf$ is uniquely minimized at $\alpha := 3\sqrt{R/n}$,   but the desired bound may be obtained by the simple choice $\alpha := 1/n$,
  and plugging the resulting Rademacher complexity estimate into
  \Cref{fact:margin_multiclass}.
\end{proof}

The proof of \Cref{fact:main:new}
now follows by instantiating \Cref{fact:main:old} for many choices of its various parameters,
and applying a union bound.
There are many ways to cut up this parameter space and organize the union bound;
the following lemma makes one such choice, whereby \Cref{fact:main:new}
is easily proved.
A slightly better bound is possible by invoking positive homogeneity of $(\sigma_1,\ldots,\sigma_L)$
to balance the spectral norms of the matrices $(A_1,\ldots,A_L)$,
however these rebalanced matrices are then used in the comparison to $(M_1,\ldots,M_L)$,
which is harder to interpret when $M_i \neq 0$.

\begin{lemma}
  \label{fact:main:new:helper}
  Suppose the setting and notation of \Cref{fact:main:new}.
  With probability at least $1-\delta$,
  every network $F_\cA : \R^d \to \R^k$ with weight matrices $\cA = (A_1,\ldots,A_L)$
  and every $\gamma > 0$
  satisfy
   \begin{align}
    &\Pr\sbr[2]{ \argmax_j F_\cA(x)_j \neq y }
    \notag\\
    &\leq
    \hcR_\gamma(F_\cA)
    + \frac 8 n
    \notag\\
    &+
    \frac{144\ln(n)\ln(2W)}{\gamma n}
    \del{\prod_i \rho_i}
    \del{1+\|X\|_2}
    \del{\sum_{i=1}^L \del{\del{\frac 1 L + \|A_i^{\top}-M_i^{\top}\|_{2,1}}\prod_{j\neq i} \del{\frac 1 L +\|A_j\|_\sigma}}^{2/3}}^{3/2}
    \notag\\
      &+
    \sqrt{\frac {9}{2n}}
    \sqrt{
      \ln(1/\delta)
      +
      \ln(2n/\gamma)
      + 2 \ln(2 + \|X\|_2)
      + 2 \sum_{i=1}^L \ln(2 + L\|A_i^{\top}-M_i^{\top}\|_{2,1})
      + 2 \sum_{i=1}^L \ln(2 + L\|A_i\|_\sigma)
    }.
    \label{eq:main:helper}
  \end{align}
 \end{lemma}
\begin{proof}
  Given positive integers $(\vec j, \vec k, \vec l) = (j_1,j_2,j_3, k_1,\ldots, k_L,l_1,\ldots,l_L)$,
  define a set of instances (a set of triples $(\gamma, X, \cA)$)
  \begin{align*}
    \cB(\vec j, \vec k, \vec l)
    &= \cB(j_1,j_2,j_3, k_1, \ldots, k_L, l_1,\ldots, l_L)
    \\
    &:=
    \cbr{
      (\gamma, X, \cA)
      \ : \ {}
      0 < \frac 1 \gamma < \frac{2^{j_1}}{n},
      \ \|X\|_2 < j_2,
      \ \|A_i^{\top}-M_i^{\top}\|_{2,1} < \frac {k_i}{L},
      \ \|A_i\|_\sigma < \frac{l_i}{L}
    }.
  \end{align*}
  Correspondingly subdivide $\delta$ as
  \begin{align*}
    \delta(\vec j, \vec k, \vec l)
    &=
    \delta(j_1,j_2,j_3, k_1, \ldots, k_L, l_1,\ldots,l_L)
    \\
    &:=
    \frac {\delta} {2^{j_1} \cdot j_2(j_2+1) \cdot k_1(k_1+1) \cdots k_L(k_L+1)\cdot l_1(l_1+1)\cdots l_L(l_L+1)}.
  \end{align*}
  Fix any $(\vec j, \vec k, \vec l)$.
  By \Cref{fact:main:old},
  with probability at least $1-\delta(\vec j, \vec k, \vec l)$,
  every $(\gamma, X, \cA) \in \cB(\vec j, \vec k, \vec l)$
  satisfies
  \begin{align}
    \Pr\sbr[2]{ \argmax_j F_\cA(x)_i \neq y }
    \ \leq\ {}
    &
    \hcR_\gamma(f)
    +
    \frac 8 n
    \notag\\
      &+
      \underbrace{
      \frac {72\cdot 2^{j_1} \cdot j_2 \ln(2W) \ln(n)} {n^2}
      \del{\prod_{i=1}^L \rho_i}
      \del{ \sum_{i=1}^L \del{ \frac{k_i}{L} \prod_{j\neq i} \frac{l_j}{L}}^{2/3} }^{3/2}
    }_{=: \heartsuit}
      \notag
      \\
    &+
    \underbrace{3\sqrt{\frac{\ln(1/\delta)
    + \ln(2^{j_1}) + 2 \ln(1+j_2) + 2 \sum_{i=1}^L \ln(1+k_i)
    + 2\sum_{i=1}^L \ln(1+l_i)}{2n}}}_{=: \clubsuit}
    .
    \label{eq:main:hi}
  \end{align}
  Since $\sum_{\vec j , \vec k, \vec l} \delta(\vec j, \vec k, \vec l) = \delta$,
  by a union bound, the preceding bound holds simultaneously over all $\cB(\vec j, \vec k, \vec l)$
  with probability at least $1-\delta$.

  Thus, to finish the proof, discard the preceding failure event,
    and let an arbitrary $(\gamma, X, \cA)$ be given.
  Choose the smallest $(\vec j, \vec k, \vec l)$ so that $(\gamma, X, \cA) \in \cB(\vec j, \vec k, \vec l)$;
  by the preceding union bound, \cref{eq:main:hi}
  holds for this $(\vec j, \vec k, \vec l)$.
  The remainder of the proof will massage \cref{eq:main:hi} into the form in the statement of \Cref{fact:main:new}.
  
  As such, first consider the case $j_1 = 1$,
  meaning $\gamma < 2/n$; then
  \[
    \Pr\sbr[2]{ \argmax_j F_\cA(x)_j \neq y }
    \leq
    1
    <
    \frac 1 {\gamma n},
  \]
  where the last expression lower bounds the right hand side of \cref{eq:main:helper},
  thus completing the proof in the case $j_1 = 1$.
  Suppose henceforth that $j_1 \geq 2$ (and $\gamma \geq 2/n$).

  Combining the preceding bound $j_2\geq 2$ with the definition of $\cB(\vec j, \vec k, \vec l)$,
  the elements of $(\vec j, \vec k, \vec l)$ satisfy
  \begin{align*}
    2^{j_1} & \leq \frac {2n} \gamma,
    \\
    j_2 & \leq 1 + \|X\|_2,
    \\
    \forall i \centerdot \quad
    k_i & \leq 1 + L\|A_i^{\top}-M_i^{\top}\|_{2,1},
    \\
    \forall i \centerdot \quad
    l_i & \leq 1 + L\|A_i\|_\sigma.
  \end{align*}
  For the term $\heartsuit$,
  the factors with $(\vec j, \vec k, \vec l)$ are bounded as
  \begin{align*}
    &2^{j_1} \cdot j_2
      \del{ \sum_{i=1}^L \del{ k_i \prod_{j\neq i} l_j}^{2/3} }^{3/2}
      \\
      &\leq
      \frac {2n} \gamma
      \del{1 + \|X\|_2}
      \del{ \sum_{i=1}^L \del{ (L^{-1}+\|A_i^{\top}-M_i^{\top}\|_{2,1})
        \prod_{j\neq i} (L^{-1}+\|A_i\|_\sigma)}^{2/3} }^{3/2}.
    \end{align*}
  For the term $\clubsuit$,
  the factors with $(\vec j, \vec k, \vec l)$ are bounded as
  \begin{align*}
    &\ln(2^{j_1})
    + 2 \ln(1+j_2)
    + 2 \sum_{i=1}^L \ln(1+k_i)
    + 2\sum_{i=1}^L \ln(1+l_i)
    \\
    &\leq
    \ln(2n/\gamma)
    + 2 \ln(2 + \|X\|_2)
    + 2 \sum_{i=1}^L \ln(2 + L\|A_i^{\top}-M_i^{\top}\|_{2,1})
    + 2 \sum_{i=1}^L \ln(2 + L\|A_i\|_\sigma).
  \end{align*}
  Plugging these bounds on $\heartsuit$ and $\clubsuit$
  into \cref{eq:main:hi}
  gives \cref{eq:main:helper}.
      \end{proof}

The proof of \Cref{fact:main:new} is now a consequence of \Cref{fact:main:new:helper},
simplifying the bound with a $\widetilde\cO(\cdot)$.
Before proceeding, it is useful to pin down the asymptotic notation $\widetilde\cO(\cdot)$,
as it is not completely standard in the multivariate case.
The notation can be understood via the $\limsup$ view of $\cO(\cdot)$; namely, $f = \widetilde \cO(g)$
if there exists a constant $C$ so that any sequence $((n^{(j)}, \gamma^{(j)}, X^{(j)}, A_1^{(j)},\ldots,A_L^{(j)}))_{j=1}^\infty$
with $n^{(j)}\to\infty$, $\gamma^{(j)}\to \infty$, $\|X^{(j)}\|_2\to \infty$, $\|A_i^{(j)}\|_1\to\infty$
satisfies
\[
  \limsup_{j\to\infty}
  \frac
  {f(n^{(j)}, \gamma^{(j)}, X^{(j)}, A_1^{(j)},\ldots,A_L^{(j)})}
  {g(n^{(j)}, \gamma^{(j)}, X^{(j)}, A_1^{(j)},\ldots,A_L^{(j)})
  \ \polylog(g(n^{(j)}, \gamma^{(j)}, X^{(j)}, A_1^{(j)},\ldots,A_L^{(j)}))}
  \leq C.
\]

\begin{proof}[Proof of \Cref{fact:main:new}]
  Let $f=f_0+f_1+f_2$ denote the three excess risk terms of the upper bound from \Cref{fact:main:new:helper},
  and $g = g_1+g_2$ denote the two excess risk terms of the upper bound from \Cref{fact:main:new};
  as discussed above,
  the goal is to show that there exists a universal constant $C$ so that
  for any sequence of tuples
  $((n^{(j)}, \gamma^{(j)}, X^{(j)}, A_1^{(j)},\ldots,A_L^{(j)}))_{j=1}^\infty$
  increasing as above,
  $\limsup_{j\to\infty} f / (g\ \polylog(g)) \leq C$.

  It is immediate that $\limsup_{j\to\infty} f_0 / g = 0$
  and $\limsup_{j\to\infty} f_1/(g_1 \ln(g)) \leq 144$.
  The only trickiness arises when studying $f_2 / (g_2\ln(g))$,
  namely the term
  $\sum_i \ln(2 + L\|A_i^{\top}-M_i^{\top}\|_{2,1})$,
  since $g_2$ instead has the term $\ln(\sum_i\|A_i^{\top}-M_i^{\top}\|_{2,1}^{2/3})$,
  and the ratio of these two can scale with $L$.
  A solution however is to compare to $\ln(\prod_i \|A_i\|_\sigma)$,
  noting that $\|(A_i)^{\top}\|_{2,1} \leq W^{1/2}\|A_i\|_2 \leq W\|A_i\|_\sigma$:
  \begin{align*}
    \limsup_{j\to\infty}
    \frac {\sum_i \ln(2+L\|(A_i^{(j)})^{\top}-M_i^{\top}\|_{2,1})}{\ln(\prod_i \|A_i^{(j)}\|_\sigma)}
    &\leq
    \limsup_{j\to\infty}
    \frac {\sum_i \ln(2+L\|(A_i^{(j)})^{\top}\|_{2,1} + L\|M_i^{\top}\|_{2,1})}{\sum_i \ln(\|(A_i^{(j)})^{\top}\|_{2,1}/W)}
    =1.
  \end{align*}
\end{proof}

\subsection{Proof of lower bound (\Cref{fact:spectral_lb})}
\begin{proof}[Proof of \Cref{fact:spectral_lb}]
  Define
  \[\mc{F}(r) := \cbr{A_L\sigma_{L-1}(A_{L-1} \cdots \sigma_2(A_2\sigma_1(A_1x)) : \prod_{i=1}^{L}\nrm*{A_{i}}_{\sigma}\leq{}r},\]
  where each $\sigma_i = \sigma$ is the ReLU and each $A_{k}\in\R^{d_{k}\times{}d_{k-1}}$, with $d_{0}=d$ and $d_{L}=1$,
  and let $S:=(x_1,\ldots,x_n)$ denote the sample.

Define a new class $\mc{G}(r)=\crl*{x\mapsto{}\tri*{a,x}\mid{}\nrm*{w}_{2}\leq{}r}$.
 It will be shown that $\mc{G}(r)\subseteq{}\mc{F}(C\cdot{}r)$ for some $C>0$,
  whereby the result easily follows from a standard lower bound on $\Rad(\mc{G}(r)_{|S})$.

Given any linear function $x\mapsto{}\tri*{a,x}$ with $\nrm{a}_{2}\leq{}r$,
  construct a network $f=A_L\sigma_{L-1}(A_{L-1} \cdots \sigma_2(A_2\sigma_1(A_1x)))$ as follows:
\begin{itemize}
\item $A_{1} = (\bfe_{1}-\bfe_{2})a^\top$.
\item $A_{k} = \bfe_{1}\bfe_1^\top + \bfe_{2}\bfe_{2}^\top$ for each $k\in\crl*{2,\ldots,L-1}$.
\item $A_{L} = \bfe_{1} - \bfe_{2}$.
\end{itemize}
It is now shown that $f(x) = \tri*{a,x}$ pointwise.
  First, observe $\sigma(A_{1}x) = (\sigma(\tri*{a,x}), \sigma(-\tri*{a,x}),0,\ldots,0)$.
  Since $\sigma$ is positive homogeneous,
  $\sigma_{L-1}(A_{L_1} \cdots \sigma_2(A_2{}y) = A_{L-1}A_{L-2} \cdots A_2{}y = (y_1, y_2, 0,\ldots,0)$ for any $y$ in the non-negative orthant.
  Because $\sigma(A_1x)$ lies in the non-negative orthant, this means $\sigma_{L-1}(A_{L-1} \cdots \sigma_2(A_2\sigma_1(A_1x)))= (\sigma(\tri*{a,x}), \sigma(-\tri*{a,x}), 0,\ldots,0)$. Finally, the choice of $A_{L} = \bfe_{1} - \bfe_{2}$ gives $f(x) = \sigma(\tri*{a,x}) - \sigma(-\tri*{a,x}) = \tri*{a,x}$.

Observe that for all $k\in\crl*{2,\ldots,L-1}$, $\nrm{A_{k}}_{\sigma}=1$. For the other layers, $\nrm{A_{L}}_{\sigma}=\nrm{A_{L}}_{2}=\sqrt{2}$ and $\nrm*{A_{1}}_{\sigma}=\sqrt{2}\cdot{}r$, which implies $f\in\mc{F}(2r)$.

Combining the pieces,
\[
  \Rad(\mc{F}(2r)_{|S}) \geq{} \Rad(\mc{G}(r)_{|S}) = \En\sup_{a:\nrm*{a}_{2}\leq{}r}\sum_{t=1}^{n}\eps_{t}\tri*{a,x_{t}} = r\cdot{}\En\nrm*{\sum_{t=1}^{n}\eps_{t}x_t}_{2}.
\]
Finally, by the Khintchine-Kahane inequality there exists $c>0$ such that
  \[
    \En\nrm*{\sum_{t=1}^{n}\eps_{t}x_t}_{2}\geq{}c\cdot{}\sqrt{\sum_{t=1}^{n}\nrm*{x_{t}}_{2}^{2}} = c\nrm*{X}_{2}.
  \qedhere
  \]
\end{proof}

\end{document}